\documentclass{article}
\usepackage{graphicx} 
\usepackage{geometry}
\usepackage{amsmath, amsfonts, amssymb, amsthm, mathtools}
\usepackage{dsfont}
\usepackage{hyperref}
\hypersetup{
    colorlinks=true,
}
\usepackage{appendix}
\usepackage{subcaption}
\usepackage[ruled,linesnumbered]{algorithm2e}
\usepackage{xcolor}
\usepackage{authblk}
\usepackage[most]{tcolorbox}

\theoremstyle{plain}
\newtheorem{assum}{Assumption}
\theoremstyle{plain}
\newtheorem{lem}{Lemma}
\theoremstyle{plain}
\newtheorem{thm}{Theorem}
\theoremstyle{definition}
\newtheorem{defn}{Definition}
\theoremstyle{remark}
\newtheorem{remark}{Remark}

\usepackage[suppress]{color-edits}
\addauthor[Christina]{cy}{violet}
\addauthor[Xumei]{xx}{blue}
\addauthor[Yudong]{yc}{cyan}

\allowdisplaybreaks

\title{Entry-Specific Matrix Estimation under Arbitrary Sampling Patterns through the Lens of Network Flows}

 \author[1]{Yudong Chen}
 \author[2]{Xumei Xi}
 \author[2]{Christina Lee Yu}

 \affil[1]{Department of Computer Sciences, University of Wisconsin-Madison}
 \affil[2]{School of Operations Research and Information Engineering, 
 Cornell University}

 \date{ \small \texttt{ yudong.chen@wisc.edu, \{xx269, cleeyu\}@cornell.edu}}

\begin{document}

\global\long\def\sgn{\operatorname{sgn}}%
\global\long\def\diag{\operatorname{diag}}
\global\long\def\KL{\operatorname{D_{KL}}}
\global\long\def\path{\operatorname{path}}
\global\long\def\rank{\operatorname{rank}}%
\global\long\def\argmax{\operatorname{argmax}}%
\global\long\def\argmin{\operatorname{argmin}}%
\global\long\def\vec{\operatorname{vec}}
\global\long\def\Span{\operatorname{span}}
\global\long\def\var{\operatorname{Var}}

\global\long\def\norm#1{\left\Vert #1\right\Vert }%
\global\long\def\card#1{\left|#1\right|}%
\global\long\def\prob#1{\mathbb{P}\left[#1\right]}%
\global\long\def\expt#1{\mathbb{E}\left[#1\right]}%
\global\long\def\expect#1{\mathbb{\mathbb{E}}\left[#1\right]}
\global\long\def\opnorm#1{\left\Vert #1\right\Vert _{\mathrm{op}}}%
\global\long\def\ip#1{\left\langle #1\right\rangle }%
\global\long\def\nunorm#1{\left\Vert #1\right\Vert _{\ast}}%
\global\long\def\twonorm#1{\left\Vert #1\right\Vert _{2}}%
\global\long\def\fnorm#1{\left\Vert #1\right\Vert _{F}}%
\global\long\def\inftynorm#1{\left\Vert #1\right\Vert _{\infty}}%
\global\long\def\twotoinftynorm#1{\left\Vert #1\right\Vert _{2 \to \infty}}%
\global\long\def\infnorm#1{\left\Vert #1\right\Vert _{\infty}}%

\global\long\def\R{\mathbb{R}}%
\global\long\def\Z{\mathbb{Z}}%
\global\long\def\pmin{p_{\min}}%
\global\long\def\pthr{p_{\mathrm{thr}}}%
\global\long\def\normal{\mathcal{N}}%
\global\long\def\Su{\mathbb{S}}%
\global\long\def\ind{\mathds{1}}
\global\long\def\one{\boldsymbol{1}}
\global\long\def\zero{\boldsymbol{0}}
\global\long\def\cU{\mathcal{U}}
\global\long\def\cV{\mathcal{V}}
\global\long\def\cG{\mathcal{G}}
\global\long\def\cE{\mathcal{E}}
\global\long\def\aa{\boldsymbol{a}}
\global\long\def\bb{\boldsymbol{b}}
\global\long\def\xx{\boldsymbol{x}}
\global\long\def\uu{\boldsymbol{u}}
\global\long\def\vv{\boldsymbol{v}}
\global\long\def\ii{\boldsymbol{i}}
\global\long\def\rr{\boldsymbol{r}}
\global\long\def\cc{\boldsymbol{c}}
\global\long\def\bbeta{\boldsymbol{\beta}}
\global\long\def\indic{\mathds{1}} 
\global\long\def\ddup{\textup{d}}%

\global\long\def\lse{\textup{LSE}}%
\global\long\def\efe{\textup{EFE}}%
\global\long\def\dop{\textup{DOP}}

\maketitle

\begin{abstract}
Matrix completion tackles the task of predicting missing values in a low-rank matrix based on a sparse set of observed entries. It is often assumed that the observation pattern is generated uniformly at random or has a very specific structure tuned to a given algorithm. There is still a gap in our understanding when it comes to arbitrary sampling patterns. Given an arbitrary sampling pattern, we introduce a matrix completion algorithm based on network flows in the bipartite graph induced by the observation pattern. For additive matrices, the particular flow we used is the electrical flow and we establish error upper bounds customized to each entry as a function of the observation set, along with matching minimax lower bounds. Our results show that the minimax squared error for recovery of a particular entry in the matrix is proportional to the effective resistance of the corresponding edge in the graph. Furthermore, we show that our estimator is equivalent to the least squares estimator. We apply our estimator to the two-way fixed effects model and show that it enables us to accurately infer individual causal effects and the unit-specific and time-specific confounders. For rank-$1$ matrices, we use edge-disjoint paths to form an estimator that achieves minimax optimal estimation when the sampling is sufficiently dense. Our discovery introduces a new family of estimators parametrized by network flows, which provide a fine-grained and intuitive understanding of the impact of the given sampling pattern on the relative difficulty of estimation at an entry-specific level. This graph-based approach allows us to quantify the inherent complexity of matrix completion for individual entries, rather than relying solely on global measures of performance. 
\end{abstract}

\section{Introduction}
\label{sec:intro}

High-dimensional data appears in a broad variety of settings, including e-commerce, business operations, social networks, biological research, and medical records among others. 
Before this data can be used to aid in decision-making, it often needs to be cleaned, which may involve denoising the dataset and predicting missing entries in the data.
This enables one to estimate underlying patterns in the data given measurements that can contain significant amount of noise, errors, and missing data. 
A rich literature in matrix completion and matrix estimation has been developed over the past 15 years to address this challenging and important task. 
Matrix completion studies the task of estimating a low-rank matrix $M^* \in \R^{n \times m}$ given potentially noisy observations of a subset $\Omega \subseteq [n]\times[m]$ of its entries~\cite{davenport2016overview, chen2018harnessing}. 
It is a powerful tool in various areas such as collaborative filtering~\cite{rennie2005cf}, system identification~\cite{liu2010systemid}, sensor localization~\cite{biswas2006sensor}, and causal inference~\cite{athey2021completion}.
The developed algorithms have been widely adopted in a variety of industries, demonstrating that these methods are truly practical even when theoretical assumptions may be violated.     

The basic task of matrix completion is summarized as follows:
We observe a noisy and sparsely sampled data matrix
\begin{equation}
M = \Omega \circ (M^\ast + E),
\end{equation}
where $E \in \R^{n \times m}$ is mean zero additive noise, $\Omega \in \{0,1\}^{n \times m}$ is the observation matrix, and $M^\ast \in \R^{n\times m}$ is the underlying signal of interest. 
Given $M$ and $\Omega$, the goal is to output an estimate $\hat{M}$ for the signal $M^\ast$. 
As this task is impossible without imposing some assumptions on $M^\ast$, the literature typically assumes that $M^\ast$ is low rank. There is both empirical and theoretical evidence~\cite{udell2019lowrank} that the low-rank assumption is in fact commonly satisfied for large matrices.
    
Fundamental research questions include what is the minimal amount of data required to guarantee statistical recovery of the underlying true pattern governing the data? How do we design efficient and optimal algorithms? 
Numerous methods with provable guarantees have been developed, including convex relaxation~\cite{candes2010matrix, candes2012exact}, alternating minimization~\cite{koren2009matrix,jain2013alternating} and spectral algorithms~\cite{keshavan2009matrix}. In these earlier works, it is typically assumed that observations are uniformly distributed across the matrix, and the goal was to either derive conditions for exact recovery in the noiseless setting (e.g.~\cite{candes2010matrix, keshavan2009matrix}) or characterize the mean squared error averaged across entries in the presence of noise, equivalent to bounding the Frobenius norm of the error matrix $\|\hat{M} - M^\ast\|_F$ (e.g.~\cite{keshavan2010matrix}). Other works such as~\cite{Abbe2020entrywise} show uniform bounds on the maximum entrywise error $\max_{i,j} |\hat{M}_{ij} - M^\ast_{ij}|$.
Despite leading to very beautiful theoretical results, the uniform observation assumption is often violated in real-world applications due to selection biases and potential confounding. 
In addition, more fine-grained, entry-specific error bounds are sometimes needed in order to make downstream decisions that involve comparing estimates of individual entries. 

There has been a growing interest in relaxing the uniform observation requirement by considering highly non-uniform observation models~\cite{xi2023nonuniform} or deterministic ones~\cite{foucart2021weighted, heiman2014deterministic, hazan2024partial,chatterjee2020deterministic}. \cycomment{I guess there are also ones that look at nonuniform but only up to a constant \cite{bhattacharya2022matrix} ... there are some older ones too in the literature I think that use a similar idea of estimating the probability of observation and then dividing by the estimated probability.}
However, when the sampling pattern $\Omega$ is highly non-uniform, we should not expect uniform error bounds, as it may be expected that some entries can be recovered to a high accuracy while other entries cannot be recovered at all. 
A few recent works~\cite{xi2023nonuniform, agarwal2021causal} consider the challenge of entry-specific error bounds with non-uniform or deterministic observation models.
However, these results require strong structural assumptions on the observation pattern, as both works~\cite{xi2023nonuniform, agarwal2021causal} require one or more densely/completely observed submatrices of adequate sizes. In a generic dataset where this structure is missing, their methods fall short as their error bounds can be of constant order. 

We focus on the task of entry-specific estimation under completely arbitrary observation patterns:

\begin{tcolorbox}[center,colback=gray!5!white,colframe=gray!75!black,size=normal,halign=flush center,fontupper=\itshape]
{\bf entry-specific Matrix Estimation \textemdash} Given any arbitrary unstructured observation set $\Omega$, can we output an estimate $\hat{M}$ such that for each entry $(i,j)$, the performance matches the entry-specific minimax optimal error? How does the entry-specific minimax optimal error depend on $(i,j)$ and $\Omega$?
\end{tcolorbox}

An estimation algorithm that achieves the entry-specific minimax optimal error guarantees for a given dataset, regardless of the observation pattern $\Omega$, effectively extracts the maximal amount of information from the data. These entry-specific guarantees enhance downstream decision-making tasks and can also indicate whether additional data collection is necessary. \xxedit{Furthermore, entry-specific guarantees open up new opportunities for optimizing the sampling process itself. In practice, this can lead to more efficient data collection strategies, where resources are allocated based on the specific challenges posed by the sampling pattern. }
Several recent works~\cite{agarwal2021causal, xi2023nonuniform} consider entry-specific guarantees. However, our work is the first to derive such guarantees for arbitrary observation patterns. 

With all conditions on the observation set $\Omega$ removed, it may be unclear where to start, as the previous literature either relied heavily on uniform-like sampling through spectral properties of random matrices or expander graphs or used a completely/densely observed large submatrix as an anchor for estimations.
In fact, ~\cite{hardt2014computational} shows that matrix completion becomes computationally intractable without distributional assumptions on the observation pattern, even if we observe 90\% of the entries, the output matrix can have a higher rank, and the underlying signal matrix is incoherent.
\xxcomment{\cite{hardt2014computational} claims "Matrix Completion
remains computationally intractable even if the unknown matrix has rank 4 but we are
allowed to output any constant rank matrix, and even if additionally we assume that the
unknown matrix is incoherent and are shown 90\% of the entries." Not sure if this is the hardness result we want since they didn't do entry-specific estimation.
}

As a step towards the broader goal of entry-specific matrix estimation, we consider two special subclasses of low-rank matrices: additive matrices and rank-1 matrices, introduced in Subsection~\ref{subsec:intro_additive} and~\ref{subsec:intro_rank1}, respectively.
For these, we provide entry-specific guarantees under arbitrary observation patterns. The algorithms are constructed from network flows, and the entry-specific guarantees are closely tied to graph connectivity metrics as quantified through the effective resistance and minimum cut. 
Our guarantees are conditioned on the observation pattern $\Omega$, allowing for correlation between $\Omega$ and the underlying ground truth matrix, with the key assumption that the observation noise remains independent of $\Omega$.

\subsection{Additive Matrices}
\label{subsec:intro_additive}

    \begin{defn}[Additivity]
        \label{defn:additive_rank1}
        A matrix $M^\ast \in \R^{n\times m}$ is an \emph{additive} matrix (or satisfies additivity) if there exist vectors $a^\ast \in \R^{n}$ and $b^\ast\in\R^{m}$ such that
        \begin{equation}
        \label{eq:additive_rank1_defn}
            M^\ast_{ij} = a_i^\ast + b_j^\ast,\quad i \in [n], \; j \in [m].
        \end{equation}
    \end{defn}
    
    While such additive structure may seem restrictive, these models are in fact widely used in economics and social science, referred to as two-way additive fixed effect models~\cite{imai2021fixed}, where the outcome $M_{ij}^\ast$ is the sum of a row fixed effect $a_i^\ast$ and a column fixed effect $b_j^\ast$.
    
    \paragraph{Panel Data Analysis under Two-Way Fixed Effects Model.} We will specifically discuss the implications of our results for estimating causal effects from panel data under a heterogeneous two-way additive fixed effects model. Suppose we have a partially observed panel dataset of $N$ units and $T$ time periods, where a unit $i$ at time $t$ may be exposed to treatment or control, and the treatment indicator variable is denoted by $X_{it} \in \{0,1\}$. 
    Let $\Omega \in \{0,1\}^{N \times T}$ indicate the observed entries. Therefore, $\Omega^0 = \Omega \circ (1-X)$ denotes the entries observed under control, and $\Omega^1 = \Omega \circ X$ denotes the entries observed under treatment.
    
    \begin{defn}
    \label{defn:fixed_effects}
        The outcomes matrix $Y$ satisfies the \emph{heterogeneous two-way fixed effects model} if
        \begin{align}
        \label{eq:linear_fixed_effect}
            Y_{it} =  \alpha_i + \gamma_t + \beta_{it} X_{it} + E_{it}, \;\text{ and } \beta_{it} = a_i + b_t, \quad\text{ for } i \in[N], \; t\in[T],
        \end{align}
        where $X_{it} \in \{0, 1\}$ indicates the binary treatment, $Y_{it}$ is the outcome, $\alpha_i$ and $\gamma_t$ are the unit and time fixed effects, $\beta_{it}$ denotes the individual causal effects, and $E_{it}$ is a mean-zero noise term.
    \end{defn}
    
    Looking at~\eqref{eq:linear_fixed_effect}, we notice that for the control group, the outcome admits a simple expression
    \begin{align}
    \label{eq:Y_control_decomp}
        Y_{it} = \alpha_i + \gamma_t + E_{it}, \quad \text{ if } X_{it}=0.
    \end{align} 
    For the treatment group, we have a similar expression
    \begin{align}
    \label{eq:Y_treatment}
        Y_{it} = ( \alpha_i + a_i ) + ( \gamma_t + b_t)  + E_{it}, \quad \text{ if } X_{it}=1.
    \end{align}
    The individual causal treatment effect is the difference in outcomes under treatment and control, which is equal to $\beta_{it} = a_i + b_t$ for this model. We can estimate $\beta_{it}$ by taking the difference between predictions of the outcome of unit $i$ at time $t$ under treatment and control. As the outcomes under treatment and control both satisfy additivity, the causal estimation task reduces to matrix completion on additive matrices corresponding to the treatment dataset and the control dataset. The fixed effect assumption is common in economics and the two-way fixed effects regression is considered as the default method to estimate causal effects from panel data, as mentioned in~\cite{imai2021fixed}. The model in Definition \ref{defn:fixed_effects} generalizes beyond the standard two-way fixed effects model by allowing for heterogeneity in the causal effects $\beta_{it}$. 
    
\subsection{Rank-1 Matrices}
\label{subsec:intro_rank1}

    \begin{defn}[Rank-1]
        \label{defn:rank1}
        A matrix $M^\ast \in \R^{n\times m}$ is rank-1 if there exist vectors $a^\ast \in \R^{n}$ and $b^\ast\in\R^{m}$ such that
        \begin{equation}
        \label{eq:latent_variable}
            M^\ast_{ij} = a_i^\ast  b_j^\ast,\quad i \in [n], \; j \in [m].
        \end{equation}
    \end{defn}
    
    The rank-$1$ matrix completion task arises in a few engineering and social data applications which we discuss below.
            
    \paragraph{$\Z_2$-synchronization} Suppose we have multiple unknown labels from $\Z_2 = \{\pm 1\}$ and we want to recover them from noisy pairwise observations. This problem is a simple example of more general models such as phase synchronization and $\mathrm{SO}(3)$-synchronization, with applications in cyro-EM~\cite{shkolnisky2012viewing} and calibration of cameras~\cite{tron2009distributed}. 
    Specifically, consider a signal vector $ x \in \{\pm 1\}^n$. Let the observation matrix be $Y = xx^\top + \sigma W$. The noise matrix $W \in \R^{n \times n}$ satisfies: $W_{ij} = W_{ji}$, $W_{ii}=0$, and $\{W_{ij}, i<j\}$ are independent $\normal(0,1)$ variables. 
    Note that there is no missingness pattern associated with the observation $Y$ so the main task is to denoise $Y$. It has been proved that the information-theoretic threshold for exact recovery is $\sigma \lesssim \sqrt{\frac{n}{\log n}}$, as established by~\cite{bandeira2017tightness}. 
    Common estimators for this recovery problem include the maximum likelihood estimator~(MLE) and the spectral estimator.
    
    \paragraph{Crowdsourcing} Crowdsourcing refers to the approach of data collection for tasks that require human expertise such as image recognition. We assign different tasks to workers with varying competence levels to obtain binary or multi-class labels. The goal is to estimate the unknown ground truth labels, which could be challenging, especially when workers' reliability levels are highly variable. A classic model for crowdsourcing is the Dawid \& Skene (DS) single coin model~\cite{dawid1979maximum}. In this model, each worker makes mistakes independently from other workers. The probability of worker $i$ making a mistake is task-independent and only depends on $i$. 
    According to~\cite{berend2014consistency}, estimating the competence levels of the workers is an essential step to optimal estimation of the ground truth labels. This step can be framed as a rank-$1$ matrix completion problem~\cite{ma2020gradient, ma2020adversarial}. 
    Suppose that tasks are binary, and let $q^*_j \in \{\pm 1\}$ be the correct answer to the $j$th question. Suppose the $i$th worker correctly answers a question with probability $p^*_i$. Then the expectation of worker $i$'s answer to task $j$ is $p^*_iq^*_j + (1-p^*_i)(-q^*_j) = (2p^*_i - 1) q^*_j$, which is rank-1.

    \paragraph{Preference/ranking model} The Bradley-Terry model~\cite{bradley1952rank} is a model predicting the outcomes of pairwise comparisons. For example, it can be used to predict the probability that a sports team $i$ wins a game against team $j$ via $p_{ij}=\prob{i>j} = \frac{p_i}{p_i + p_j}$, where $p_i$ represents a score measuring the strength or quality of $i$. The exponential score function $p_i = e^{\beta_i}$ indicates that the odds of the preference probabilities, $\frac{p_{ij}}{1-p_{ij}} = e^{\beta_i} / e^{\beta_j}$, form a rank-$1$ matrix.

\subsection{Summary of Contributions}

    We summarize our contributions as follows:

    \medskip \noindent 
    1) We propose network flow-based algorithms for estimating additive or rank-$1$ matrices under arbitrary observation patterns. The algorithms reliably predict missing values with \emph{entry-specific} estimation guarantees, that align with the minimax lower bounds under certain conditions.
        Our theoretical guarantees provide insight into the relative estimation hardness of the missing entries in the data. Our bounds show that entries corresponding to vertices with stronger connectivity in the bipartite graph formed by the observations exhibit smaller estimation errors. This finding underscores the connection between matrix completion and graph theory, potentially opening up new areas for exploration.
    
    \medskip \noindent 
    2) For additive matrices, our estimator is based on the electrical flow derived from viewing the bipartite graph as an electrical network and applying unit current between the vertices of interest. 
            \begin{itemize}
                \item We show that the entry-specific error upper bounds match the local minimax lower bounds up to logarithmic factors. In addition, the estimation variance, governed by the effective resistance between the corresponding vertices, reflects how graph connectivity affects estimation quality, offering a clear view of estimation difficulty and paving the way for future research.
                \item We show that our electrical flow-based estimator is equivalent to the least squares estimator, despite their different approaches—entry-level versus global optimization. This connection provides an intuitive interpretation of entry-specific estimation quality of the least squares estimators and reveals that the least squares estimator achieves a minimax lower bound for each entry. Moreover, this equivalence confirms that our electrical flow-based estimator is the uniformly minimum-variance unbiased estimator~(UMVUE).
                \item When applied to two-way fixed effects models, our algorithm enables practitioners to predict \emph{individual} causal effects instead of an aggregated one, with provable entry-specific guarantees. It also establishes the necessary and sufficient conditions for identifying these individual causal effects.
                We further uncover a surprising yet intuitive connection between difference-in-differences (DiD) estimators and TWFE regression: both are part of a family of estimators parameterized by flows over networks.
                 Specifically, restricting flows to length-3 paths yields DiD, while electrical flow corresponds to a variant of the TWFE regression estimator.
                 \item Our proof for local minimax optimality provides a novel perspective through voltages in the electrical network. We discover that the dual variables for optimizing the variance among all unit flows correspond precisely to these voltages. This technique can be generalized to linear regression problems, providing an alternative method for constructing hard instances to prove minimax optimality.
            \end{itemize}
    
    \medskip \noindent 
    3) For rank-$1$ matrices, our estimator uses network flows constructed by edge-disjoint paths in the graph.
            \begin{itemize}
                \item The entry-specific error upper bounds again depend on the connectivity of the vertices of interest. Specifically, they are affected by the maximum number of edge-disjoint paths (equal to the size of the minimum edge cut) and the maximum length of these paths. When there are more and shorter paths between certain vertices, the estimation quality of the corresponding entry becomes better. We further show that the estimator is minimax optimal when the observation is sufficiently dense or has a certain structure.
                \item Similarly, we present a novel construction using minimum cut to show minimax optimality. This approach not only provides a fresh perspective but also clarifies how the estimation quality depends on minimum cuts.
            \end{itemize}

    \paragraph{Notation.}
    We use $c,C$ etc.\ to denote positive absolute constants, which might change from line to line. Let $[n]:=\{1,2,\dots,n\}$.
    For non-negative sequences $\{f_n\}$ and $\{g_n\}$, we write $f_n \lesssim g_n$ when $f_n \le C g_n,\forall n$, and write $f_n \asymp g_n$ or $f_n = \Theta(g_n)$ when both $f_n \lesssim g_n$ and $f_n \gtrsim g_n$ hold. While $\Omega$ represents a binary matrix, we may also refer to $\Omega$ as a subset consisting of the index pairs $(i,j)$ for which $\Omega_{ij} = 1$.
\section{Related Literature}

\paragraph{Matrix Completion: Beyond Uniform Sampling.}
    The problem of matrix completion studies estimating a low-rank matrix based on potentially noisy and partial observations of its entries~\cite{davenport2016overview, chen2018harnessing}. 
    Earlier research~\cite{candes2010matrix,koren2009matrix, keshavan2009matrix} typically assumes that observations are uniformly distributed across the matrix, meaning each entry has an equal chance of being observed and that sampling is done independently across entries. The goals of estimation differ between noiseless and noisy observations. 
    In the noiseless case, we focus on deriving conditions for exact recovery (e.g.,~\cite{candes2010matrix, keshavan2009matrix}), while in the noisy scenario, the aim is to characterize an aggregate error metric, often the mean squared error averaged across entries (e.g.,~\cite{keshavan2010matrix}). 
    However, the assumption of uniform observation is often violated in practical applications due to selection biases and potential confounding. 
    Additionally, for decision making that involves comparing individual entry estimates, more detailed, entry-specific/entrywise error bounds may be necessary alongside the mean squared error.
    
    There has been a growing interest in relaxing the uniform observation requirement by considering deterministic ones~\cite{agarwal2021causal, foucart2021weighted, heiman2014deterministic, kanade2022partial} or highly non-uniform observation models~\cite{xi2023nonuniform}.
    Other works, such as~\cite{Abbe2020entrywise}, provide more fine-grained error bounds that identify the worst-case estimation quality, under uniform sampling. 
    More recently, quite a few works~\cite{xi2023nonuniform, agarwal2021causal} successfully address both aspects at the same time, showing entry-specific error bounds with non-uniform or deterministic observation models.
    However, there remains a presumed structure in the observation pattern, since both works~\cite{xi2023nonuniform, agarwal2021causal} require one or more densely/completely observed submatrices of adequate sizes. In a generic dataset lacking such structure, these methods may fall short, as their error bounds could be of constant order. 
    Our work tackles the challenge of deriving optimal and meaningful entry-specific error bounds under arbitrary, potentially unstructured observation patterns. Addressing this challenge can lead to more robust and broadly applicable methods, especially in contexts where structured observation patterns are absent.

\paragraph{Panel Data: Heterogeneity and Arbitrary Pattern.}
    One real-world application of matrix completion is estimating treatment effects using panel data in causal inference. In such cases, the observation and treatment patterns are far from uniform and may be correlated with the observation matrix. It is no longer reasonable to rely on the uniform sampling assumption and methods that are effective under more general settings are needed. Traditional panel data methods often assume simple block or staggered treatment patterns~\cite{athey2021completion, athey2022design}, with blocks serving as a foundation for further estimation. 
    In reality, the observation and treatment patterns can be highly irregular and may lack block structures. 
    Another common practice is to focus on estimating average effects, which overlooks the underlying heterogeneity in the data. A recent line of work~\cite{goodman2021difference, callaway2021difference, sun2021estimating, de2023two, imbens2021staggered, jochmans2019fixed} allows heterogeneous treatment effects and assumes certain structures of the effects. 
    In Subsection~\ref{subsec:panel_data}, we address the challenge of arbitrary treatment patterns and heterogeneous treatment effects, uncovering a surprising yet intuitive connection between DiD estimators and TWFE regression. 
    There has been a flurry of work~\cite{wooldridge2021two, goodman2021difference, ruttenauer2024can} drawing connections to better understand TWFE regression. However, we are the first to interpret these estimators using network flows, which carries physical meaning and paves the way for future exploration. 
    There has also been work discussing how graph connectivity affects estimation. For instance, \cite{jochmans2019fixed} shows that graph connectivity as characterized by the spectral gap and overlap of direct neighbors is key to the estimation accuracy of TWFE regression for fixed effect models. Our result further gives an explicit and exact equivalence of the variance of the best estimator with the effective resistance, providing fine-grained entry-specific results.

\cycomment{Anything notable to comment on regarding rank 1 matrices? In the rank 1 setting ... have there been any results that go beyond uniform sampling?}

\section{Entry Specific Estimation for Additive Matrices}

We first focus on the estimation task for additive matrices, as introduced in Definition~\ref{defn:additive_rank1}.
We introduce a family of unbiased estimators parameterized by flows on an undirected bipartite graph constructed from the observation pattern $\Omega$. We introduce the estimator intuitively via an example of a flow consisting of a single path, and we subsequently show that the variance-minimizing estimator will use the electrical flow.

\subsection{Graph Construction}
\label{subsec:additive_graph_construction}
    For a fixed $\Omega$, we construct an undirected bipartite graph $\cG(\Omega) = (\cV, \cE(\Omega))$ where the edges in the graph are given by the sparsity pattern of $\Omega$.
    The left vertices are denoted as $\{u_i\}_{i \in [n]}$, corresponding to the $n$ rows, and the right vertices as $\{v_j\}_{j \in [m]}$, representing the $m$ columns. 
    The vertex set is therefore $\cV = \{u_i\}_{i \in [n]} \cup \{v_j\}_{j \in [m]}$.
    The vertex $u_i$ is adjacent to $v_j$  if and only if $\Omega_{ij} = 1$. 
    Hence, the edge set is $\cE(\Omega) = \{ ( u_i, v_j ): \Omega_{ij} = 1 \}$. See Figure~\ref{fig:bipartite} for illustrations of the bipartite graphs constructed from given observation patterns.
    We define $n_v = |\cV| = n+m$ and $n_e \equiv n_e(\Omega) = |\cE(\Omega)| = \one^\top \Omega \one$ as the number of vertices and edges in $\cG(\Omega)$. 
    For simplicity of presentation, we assume the vertices have the following ordering: $u_1, u_2, \dots, u_n, v_1, v_2, \dots, v_m$, i.e.~the left vertices are ordered before the right ones. 
    We will use the notation $u_i, v_j$ and their respective ordering $i, n+j$ interchangeably. 
    Under this ordering, the adjacency matrix of the bipartite graph $\cG(\Omega)$ is
    \begin{align*}
        A(\Omega) = 
        \begin{pmatrix}
            0 & \Omega \\
            \Omega^\top & 0 
        \end{pmatrix} \in \{0,1\}^{n_v \times n_v}.
    \end{align*}
    
    Without loss of generality, we assume that the graph $\cG(\Omega)$ is connected. If it is not connected, this means that $\Omega$ can be partitioned into diagonal block submatrices with zeros on the off-diagonal blocks. It is easy to argue that in such a setting the entries in the off-diagonal blocks cannot be identified, and thus it is sufficient to focus on estimation within each of the connected components.

    \begin{figure}
     \centering
     \begin{subfigure}[b]{0.4\textwidth}
         \centering
\includegraphics[width=.7\textwidth]{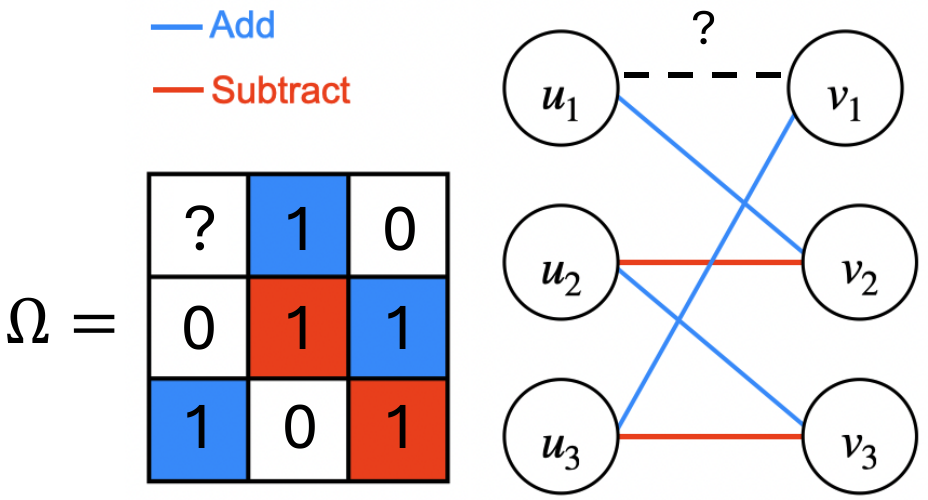}
         \caption{ \qquad \qquad \qquad}
         \label{fig:bipartite_add_subtract}
     \end{subfigure}
     \qquad
     \begin{subfigure}[b]{0.4\textwidth}
         \centering
    \includegraphics[width=.75\textwidth]{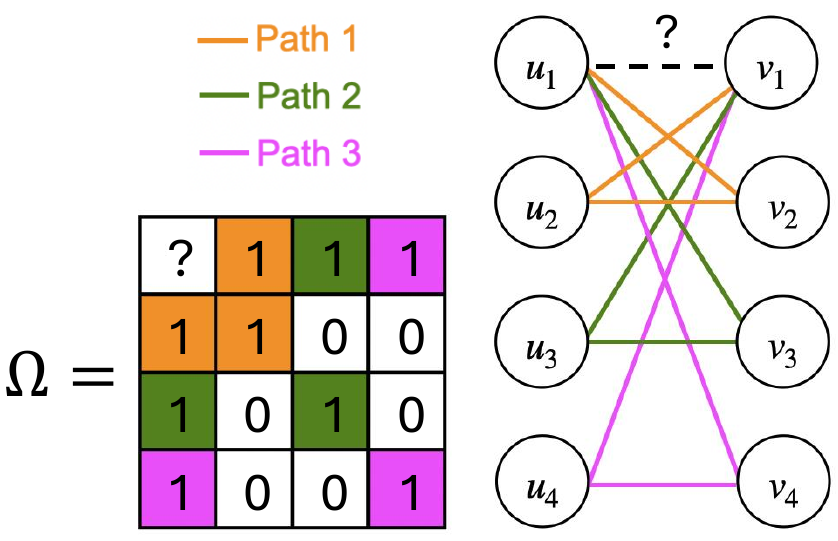}
         \caption{\qquad \qquad}
         \label{fig:bipartite_multi_path}
     \end{subfigure}
    \caption{We depict two bipartite graphs constructed from two given observation patterns. To estimate the entry $(1,1)$, we use data along the paths connecting $u_1$ and $v_1$ in the graph. In~(a), there is a path of length $\ell=5$ from $u_{1}$ to $v_{1}$. We construct an estimate by alternatingly adding and subtracting observations along the path, where blue edges indicate addition and red edges indicate subtraction. Due to the alternating signs, the latent factors corresponding to the intermediate vertices cancel out, resulting in an expected estimate of $a_1^\ast + b_1^\ast$.  
    In~(b), we have three short paths of length $3$ connecting $u_1$ and $v_1$. The orange path (path $1$) corresponds to observations in entries $(1,2), (2,2)$ and $(2,1)$; the green path (path $2$) corresponds to observations in entries $(1,3), (3,3)$ and $(3,1)$; the magenta path (path $3$) corresponds to observations in entries $(1,4), (4,4)$ and $(4,1)$.}
    \label{fig:bipartite}
    \end{figure}
    
\subsection{Network Flow Estimator} \label{sec:algo}

    Consider a concrete example depicted in Figure~\ref{fig:bipartite_add_subtract},
    where we want to estimate entry $(1,1)$ in a $3$-by-$3$ matrix, with the given $\Omega$. There is a single simple path connecting $u_1$ and $v_1$ in the corresponding bipartite graph. An unbiased estimate of $M_{11}^\ast$ can be constructed by alternating between adding and subtracting the observations on this path: 
    \[\hat M_{11} =  M_{12} - M_{22} + M_{23} - M_{33} + M_{31},\]
    where its expectation is given by 
    \begin{align*}
    \expt{\hat M_{11}} &= (a_1^\ast + b_2^\ast) - (a_2^\ast + b_2^\ast) + (a_2^\ast + b_3^\ast) - (a_3^\ast + b_3^\ast) + (a_3^\ast + b_1^\ast) = a_1^\ast + b_1^\ast = M_{11}^\ast.
    \end{align*} 
    This simple example illustrates the key intuition of the estimator. Fix $(i,j) \in [n] \times [m]$. 
    Given any path from $v_i$ to $u_j$ in the network, an unbiased estimator for $M^\ast_{ij}$ can be constructed by alternating adding and subtracting the observations on edges along the path. 
    Consider a simple path in $\cG(\Omega)$ connecting $u_{i}$ and $v_{j}$, denoted as $\path(\xx, \ell)$: $u_{x_{1}} (= u_{i}) \to v_{x_2} \to u_{x_3} \to v_{x_4} \to \dots \to u_{x_{\ell}} \to v_{x_{\ell+1}} (=v_{j})$, with path length $\ell \ge 1$. The difference-on-path estimator given by $\path(\xx, \ell)$ is
    \begin{align}
    \label{eq:single_path_est}
        \hat{M} _{ij}^{\path(\boldsymbol{x}, \ell)} = \sum_{s=1, 3, \dots, \ell} M_{x_{s} x_{s+1}} - \sum_{s=3,5,\dots,\ell} M_{x_{s} x_{s-1}}.
    \end{align}
    When $\expt{E_{ij}}=0$ for all $i,j$, the above estimator is unbiased since 
    \begin{align*}
        \expt{\hat{M} _{ij}^{\path(\boldsymbol{x}, \ell)} } = \sum_{s=1,3,\dots,\ell} ( a_{x_s}^\ast +b_{x_{s+1}}^\ast) - \sum_{s=3,5,\dots,\ell} (a_{x_s}^\ast+ b_{x_{s-1}}^\ast) = a_{i}^\ast + b_{j}^\ast = M^\ast_{ij}.
    \end{align*}
    In fact, when the noise terms are i.i.d.~$\normal(0,\sigma^2)$, the random variable $\hat{M} _{ij}^{\path(\boldsymbol{x}, \ell)} - M_{ij}^\ast$ is Gaussian with zero mean and variance $\ell \sigma^2$. This implies longer paths result in an estimator with a higher variance, because~\eqref{eq:single_path_est} will involve adding and subtracting more observations.
    The properties of the difference-on-path estimator critically rely on the additive structure of the underlying matrix $M^\ast$. 
    
    When there are multiple paths in the graph, it is natural to aggregate the estimates obtained from each path to make the most out of available information. We can achieve that by considering network flows, as any unit flow on the graph from $u_{i}$ to $v_{j}$ can be decomposed as a convex combination of paths from $u_{i}$ to $v_{j}$.
    Consequently, we can construct an unbiased estimator for a unit flow by weighting the individual path estimators accordingly.
    Let $f \in \R^{n_e}$ be an arbitrary unit flow on $\cG(\Omega)$ from $u_i$ to $v_j$. The flow is defined on the edges in $\cE(\Omega)$, where we order row vertices before column vertices, such the flow on edge $(u_i, v_j)$ denoted by $f_{i, n+j}$ will be positive if the flow goes from $u_i$ to $v_j$, and negative if the flow goes from $v_j$ to $u_i$. The \emph{unit flow estimator} constructed from $f$ is given by
    \begin{align}
    \label{eq:unit_flow_est}
    \hat M^f_{ij} = \sum_{(k, \ell) \in [n] \times [m]}  f_{k, n+\ell} \Omega_{k \ell}  M_{k \ell},
    \end{align}
    where the estimator is a weighted sum of the observations, the weight of which is assigned using $f$. Note that due to the orientation of edges from rows to columns, the flow vector will naturally alternate in signs, similar to the difference-on-path estimator in \eqref{eq:single_path_est}. The unit flow estimator is equivalent to taking a convex combination of the individual difference-on-path estimates that arise from decomposing the flow into a convex combination of paths. As a result, the unit flow estimator is also unbiased as long as the observation noise is zero-mean.
    
    Intuitively, flows that utilize more paths should result in a better estimate as they incorporate more observations. However, determining the optimal way to combine estimates from paths of varying lengths or with overlapping edges may not be immediately obvious. Assigning equal weights to paths of different lengths is not optimal since shorter paths contain less noise and should therefore be given more weights. For edges that appear in multiple paths, we should also assign more weights to them as they contribute more to the aggregate estimation. These intuitions are realized by the electrical flow, as demonstrated in Figure~\ref{fig:path_varying_length_overlapping_edge}.
    
    \begin{figure}
        \centering
        \begin{subfigure}[b]{0.45\textwidth}
        \centering
            \includegraphics[width=.4\textwidth]{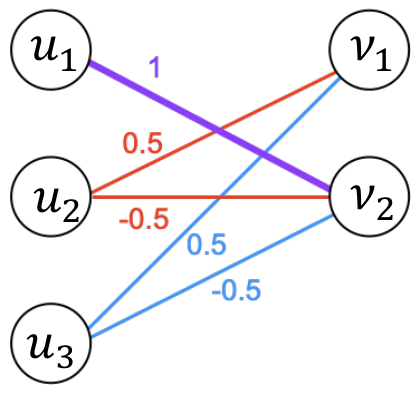}
            \caption{~}
            \label{fig:path_overlapping_edge}
        \end{subfigure}
        \begin{subfigure}[b]{0.45\textwidth}
        \centering
            \includegraphics[width=.3\textwidth]{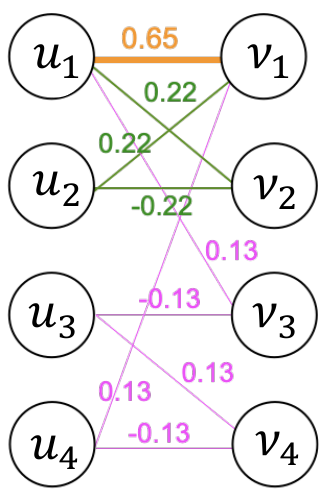}
            \caption{~}
            \label{fig:path_varying_length}
        \end{subfigure}
        \caption{Examples depicting electrical flow over graphs that contain overlapping paths and paths of varying lengths. (a) Electrical flow puts higher weight on the edge that overlaps multiple paths. (b) Electrical flow puts higher weight on shorter paths relative to longer paths.}
    \label{fig:path_varying_length_overlapping_edge}
    \end{figure}

    \paragraph{Electrical Flow Estimator~(EFE).}
    Within the class of unbiased flow estimators, we can ask which flow minimizes the variance, for which there is a surprisingly simple answer. We can view the observation graph as an electrical network as depicted in Figure~\ref{fig:electrical_network},
    where each edge $e$ represents a resistor with resistance $r_e=1$. 
    For a given flow $f$, the electrical energy of that flow is in fact proportional to the variance of the corresponding unit flow estimator.
    As a result, the electrical flow estimator~(EFE), obtained by letting $f$ be equal to the unit electrical flow, is the optimal variance-minimizing flow estimator as it is the unique unit flow that minimizes the electrical energy by Thomson's principle. 
    The variance of the resulting EFE is proportional to the effective resistance\footnote{The effective resistance between $s$ and $t$ in an electrical network is equal to the potential difference that appears across $s$ and $t$ when a unit current source is applied between them. The effective resistance is small when there are many short paths between $s$ and $t$, and it is large when there are few paths between $s$ and $t$ that tend to be long.}, which measures the connectivity between $u_i$ and $v_j$~\cite{tetali1991random, klein1993er}. This connects to a rich literature in physics and computer science that studies properties of network flows, electrical flows, and effective resistance as a function of the network structure~\cite{ghosh2008minimizing}.
    If we attach a battery in the electrical network that applies a voltage difference between $u_i$ and $v_j$ equal to the effective resistance $R(u_i,v_j)$, the resulting current $\ii$ that flows from $u_i$ to $v_j$ corresponds to the unit electrical flow.
    In Section~\ref{sec:guarantees}, we will show that the electrical flow estimator achieves both the minimax lower bound and the minimum variance when the noise terms $E_{ij}$'s are i.i.d.~$\normal(0,\sigma^2)$.

    \begin{figure}
        \centering
        \includegraphics[width=0.5\linewidth]{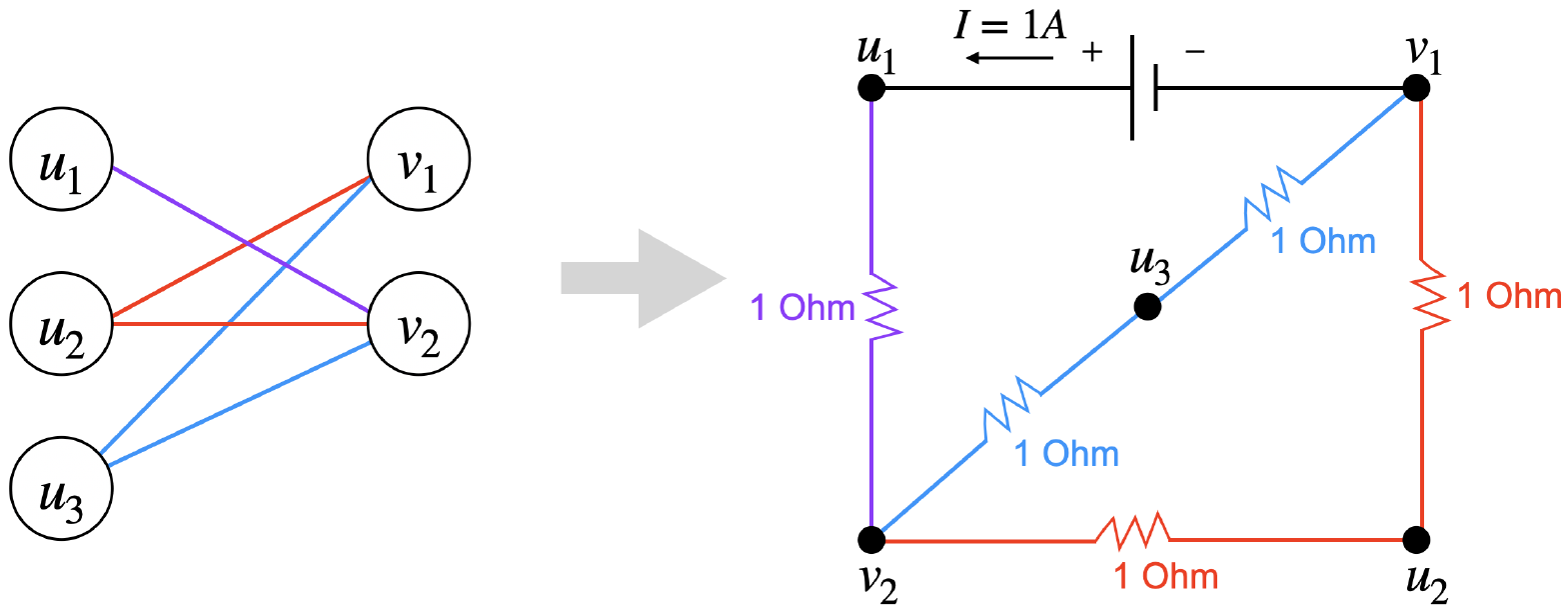}
        \caption{Electrical network constructed from a bipartite graph.}
        \label{fig:electrical_network}
    \end{figure}

    We give a summary of electrical flows and their properties in Appendix~\ref{appen:effective_resistance}.
    Applying equations~\eqref{eq:electrical_flow_defn} and~\eqref{eq:effective_resistance_defn} from Appendix~\ref{appen:effective_resistance}, 
    we can give the following explicit expression for the unit electrical current between $u_i$ and $v_j$, \cycomment{should we change to use $(s,t)$ instead of $(i,j)$ so we don't clash with the $\ii$ notation? Also if we leave the stuff in appendix, should we at least state the definition of the incidence matrix and Laplacian matrix here?}
    \begin{align}
        \ii(\Omega,i,j) = B(\Omega) L^+ (\Omega) (e_{i} - e_{n+j}), \label{eq:electrical_flow_omega_defn} 
    \end{align}
    where $B(\Omega)$ and $L^+(\Omega)$ are the oriented incidence matrix and the pseudo-inverse of the Laplacian matrix defined by the adjacency matrix $A(\Omega)$.
    Hence, the \emph{electrical flow estimator} has the form 
    \begin{align}
    \label{eq:electrical_flow_est}
        \hat M_{ij}^{\efe} &= \sum_{(k, \ell) \in [n] \times [m]}\ii(\Omega,i,j)_{k, n + \ell}   \Omega_{k\ell}  M_{k \ell}.
    \end{align}
    Define an operator $\vec_{\Omega}(\cdot): \R^{n \times m} \to \R^{n_e}$ that maps a data matrix $M$ with missing entries to a vector composed of only the observations, i.e.~$\vec_{\Omega}(M) = (M_{k \ell})_{(u_k, v_\ell) \in \cE(\Omega)}$ in row-major order.
    Under this notation, we obtain a matrix-vector formula of the electrical flow estimator:
    \begin{align}
        \hat M_{ij}^{\efe} &= \langle \ii (\Omega,i,j) , \vec_{\Omega}(M)  \rangle = \langle B(\Omega) L^+(\Omega) (e_i - e_{n+j}) , \vec_{\Omega}(M)\rangle. \label{eq:closed_form_flow_est}
    \end{align}
    The computation complexity of $\hat M_{ij}^{\efe}$ is determined by computing $L^+(\Omega)$, which is $O((n+m)^3)$. 
    
    Next, we introduce an equivalence result connecting the electrical flow estimator to the least squares estimator, which provides simpler computation compared to~\eqref{eq:closed_form_flow_est}. \cyedit{This equivalence hints at the generalization of our result to linear regression tasks.}
    The least squares estimator solves the following convex problem to estimate the latent factors:
    \begin{align}
         \label{eq:lse_factor_defn}
            (\hat a, \hat b) = \underset{(a, b) \in \R^{n_v} }{\argmin} \; f(a, b) \coloneqq \norm{\Omega \circ (a \one_n^\top + \one_m b^\top - M)}_F^2.
    \end{align}
    The following Lemma gives a closed-form expression for the least squares estimator with the minimum Euclidean norm. We partition the pseudo-inverse of the Laplacian into $4$ blocks:
    \begin{align}
    \label{eq:pinv_Laplacian_partition}
        L^+(\Omega) = \begin{pmatrix}
                \Gamma_{11} & \Gamma_{12} \\
                \Gamma_{21} & \Gamma_{22}
            \end{pmatrix},
    \end{align}
    where $\Gamma_{11} \in \R^{n \times n}, \Gamma_{12} \in \R^{n \times m} , \Gamma_{21} \in \R^{m \times n}$, and $\Gamma_{22}\in \R^{m \times m}$. The proof of Lemma~\ref{lem:closed-form_lse} is shown in Appendix~\ref{appen:proof_closed-form_lse}.
    \begin{lem}
        \label{lem:closed-form_lse}
        For $(\hat a, \hat b) \in \R^{n_v}$ satisfying~\eqref{eq:lse_factor_defn} with minimum Euclidean norm, we have
        \begin{align}
        \label{eq:closed_form_lse}
            (\hat a, \hat b) = \begin{pmatrix}
                \Gamma_{11} & -\Gamma_{12} \\
                -\Gamma_{21} & \Gamma_{22}
            \end{pmatrix}
            \begin{pmatrix}
                \diag(\Omega M^\top) \\
                \diag(\Omega^\top M)
            \end{pmatrix},
        \end{align}
        where $\diag(\cdot)$ denotes the vector formed by the diagonal of a matrix. 
    \end{lem}
    Having obtained $\hat a$ and $\hat b$, the estimator of $M^\ast$ is $\hat M ^{\lse} = \hat a \one_n^\top + \one_m \hat b^{ \top}$ and subsequently,
    \begin{align}
        \label{eq:lse_entry_defn}
        \hat{M}^{\lse}_{ij} = \hat a_i + \hat b_j.
    \end{align}
    Quite surprisingly, the electrical flow estimator is equivalent to the least squares estimator. The proof of Theorem~\ref{thm:equivalence_efe_lse} can be found in Appendix~\ref{appen:proof_equivalence_efe_lse}.
    
    \begin{thm}
    \label{thm:equivalence_efe_lse}
        If $u_i$ and $v_j$ are connected in $\cG(\Omega)$ for $(i,j) \in [n] \times [m]$, we have $\hat M^{\lse}_{ij} = \hat M^{\efe}_{ij}$.
    \end{thm}
    This equivalence gives us an easy way to compute the electrical flow estimator. Furthermore, it shows that our estimator has favorable statistical properties \cyedit{by leveraging the optimality of least squares for linear regression. The additive matrix structure reduces the nonconvex matrix completion task to a significantly simpler linear regression task.} 
    We summarize our method in Algorithm~\ref{algo:flow_est}. The computation complexity of Algorithm~\ref{algo:flow_est} is determined by computing $L^+$, which is $O((n+m)^3)$.
    
    \RestyleAlgo{ruled}
    \begin{algorithm}
         \SetAlgoLined
         \SetKwFunction{algo}{ElectricalFlowEstimator}{}
         \SetKwProg{myalg}{Function}{}{}
         \myalg{\algo{$M, \Omega$}}{
            Initialize $\hat{M}_{ij}^{\efe} \gets \Box$ for all $(i,j) \in [n] \times [m]$. \\
            \For{connected component $\cG'=(\cV', \cE')$ of $\cG(\Omega)$}{
                In graph $\cG'$, compute $L^{\prime +}$ and partition it according to~\eqref{eq:pinv_Laplacian_partition}. \\
                Compute row and column sums in the corresponding submatrix $(\Omega', M')$ as $\rr' \gets \diag(\Omega' M^{\prime \top} )$ and $\cc' \gets \diag(\Omega^{\prime \top} M')$. \\
                $\hat{a} \gets \Gamma_{11}' \rr' - \Gamma_{12}' \cc'$ and $\hat{b} \gets -\Gamma_{21}' \rr' + \Gamma_{22}' \cc'$. \\
                $\hat{M}' \gets \hat{a} \one^\top + \one \hat{b}^\top$; assign the corresponding submatrix of $\hat M^{\efe}$ as $\hat{M}'$.
            }
             \KwRet{$\hat{M}^{\efe}$}
         }
         \caption{Electrical flow estimator for additive model. \label{algo:flow_est}}
    \end{algorithm}

\subsection{Theoretical Guarantees}
\label{sec:guarantees}
        
    We introduce the theoretical guarantees for the electrical flow estimator presented in Algorithm~\ref{algo:flow_est}. 
    We divide the results into three layers, each imposing progressively stronger conditions on the noise. The first layer requires only zero-mean noise, which ensures the unbiasedness of the estimator. The second layer further assumes independent sub-Gaussian noise, leading to matching upper and lower bounds for the error. The third layer employs the strongest assumption of Gaussian noise and explores the concept of the uniformly minimum variance unbiased estimator~(UMVUE).

    \subsubsection{Identifiability and Unbiasedness}
        We prove an unidentifiability argument in Theorem~\ref{thm:unidentifiability} to show that estimating entries that are not connected in the graph is impossible. 
        \begin{thm}
            \label{thm:unidentifiability}
            For $u_i$ and $v_j$ that are not connected in $\mathcal{G}(\Omega)$, there exist two additive models $M^\ast$ and $M'^\ast$ such that $\Omega \circ M^\ast = \Omega \circ M'^\ast$ but $M^\ast_{ij} \neq M'^\ast_{ij}$.
        \end{thm}
        \begin{proof}
        For a fixed $\Omega$, we construct signals $M^\ast$ and $M'^\ast$ as:
        \begin{align*}
            M^\ast_{st} &= a_{s} + b_{t} \\
            M'^\ast_{st} &= y_{s} + z_{t}.
        \end{align*}
        For $a$ and $b$, we set them as $a = b = \boldsymbol{0}$ where  $\boldsymbol{0}$ represents an all-zero vector. 
        For $y$ and $z$, we define them according to a minimum edge cut for $u_i$ and $v_j$. 
        Suppose this cut partitions the nodes into two sets $\mathcal{L}$ and $\mathcal{R}$. We assume $u_i \in \mathcal{L}$ and $v_j \in \mathcal{R}$. 
        For a left node $u_s \in \mathcal{L}$, let $y_s = 0$. For a right node $v_t \in \mathcal{L}$, let $z_t = 0$.
        By definition, $M^\ast_{ij}=0$ and $M'^\ast_{ij}=-\epsilon$. 
        There is no path between two sides of the cut $\mathcal{L}$ and $\mathcal{R}$. Because $M^\ast$ and $M'^\ast$ only differ on the entries crossing the cut, we conclude that $\Omega \circ M^\ast = \Omega \circ M'^\ast$. 
        \end{proof}
        
        Next, we show the unbiasedness of the unit flow estimator under the zero-mean noise assumption. Since the electrical flow estimator and the difference-on-path estimator are special cases of the unit flow estimator, they are unbiased as well.
        \begin{assum}
        \label{assum:zero-mean_noise}
        We assume $\expt{E_{ij}}=0$ for all $(i,j)\in[n] \times [m]$.
        \end{assum}
        \begin{thm}
        \label{thm:unbiasedness}
        Under Assumption~\ref{assum:zero-mean_noise}, for an arbitrary unit flow $f$, we have $\expt{\hat M_{ij}^f} = M_{ij}^\ast$.
        \end{thm}
        \begin{proof}
        By model assumption and linearity of expectation, we have
        \begin{align*}
            \expt{\hat M_{ij}^f} &= \sum_{k, \ell} f_{k, n+\ell} M^\ast_{k\ell} 
            = \sum_{k, \ell} f_{k, n+\ell} \left( a^\ast_{k} + b^\ast_{\ell} \right) 
            = \sum_k a^\ast_{k} \sum_{\ell} f_{k, n+\ell} + \sum_{\ell} b^\ast_{\ell} \sum_k f_{k, n+\ell} 
            = a_i^\ast + b_j^\ast = M^\ast_{ij},
        \end{align*}
        where the last step uses the flow conservation constraint. 
        \end{proof}
        The unbiasedness of $\hat M_{ij}^f$ is ensured by the zero-mean noise and the value of $f$ being $1$. Note that we do not require any independence conditions for $E$. As long as $E$ is zero-mean, the unit flow estimator remains unbiased, regardless of the correlation in the noise.
    
    \subsubsection{Error Upper Bound and Minimax Lower Bound}
        Now we consider the electrical flow estimator.
        Besides being zero-mean, we additionally assume the noise terms are independent sub-Gaussian random variables. 
        \begin{assum}
            \label{assum:indepenent_subG}
            We assume $E_{ij}$'s are independent and each $E_{ij}$ is sub-Gaussian with parameter $\sigma^2$, for an unknown $\sigma^2>0$.
        \end{assum}
        To present the error upper bound, we define the effective resistance $R(u_i, v_j)$ as the potential difference between $u_i$ and $v_j$ one needs to apply to send one unit of current between them:
        \begin{align}
            R(u_i, v_j) = (e_i - e_{n+j})^\top L^+(\Omega) (e_{i} - e_{n+j}).\label{eq:effective_resistance_omega_defn}
        \end{align}
        It quantifies how well connected $u_i$ and $v_j$ are in the graph. See Appendix~\ref{appen:effective_resistance} for a more detailed discussion. Theorem~\ref{thm:flow_est_upper_bound} states that the squared error can be upper bounded by a factor of the effective resistance, which implies entry $(i,j)$ can be better estimated when $u_i$ and $v_j$ are more connected in the graph. 
        
        \begin{thm}
    \label{thm:flow_est_upper_bound}
            Under Assumptions~\ref{assum:zero-mean_noise} and~\ref{assum:indepenent_subG}, we have $\var[ \hat{M}_{ij}^{\efe}] \leq \sigma^2 R(u_i,v_j)$, \cyedit{where this is satisfied with equality if $\var[E_{ij}] = \sigma^2$ for all $(i,j)$.}
            Furthermore, with probability at least $1-\delta$,  we have
            \begin{align}
                \label{eq:flow_est_upper_bound}
                \big( \hat{M}_{ij}^{\efe} - M^\ast_{ij} \big)^2 \le 2 \sigma^2 R(u_i, v_j) \log ( 2nm / \delta), \quad \forall \text{ connected } (i,j).
            \end{align}
        \end{thm}
        
        \begin{proof}
            Fix $(i,j)$. From Theorem~\ref{thm:unbiasedness}, we know $\hat M_{ij}^{\efe}$ is unbiased for $M_{ij}^\ast$. 
            Next, we compute its variance. For an arbitrary flow $f$, the unit flow estimator $\hat M^{f}_{ij}$ is sub-Gaussian with parameter upper bounded by $\sigma^2 \sum_{k, \ell} f_{k, n+\ell}^2$. 
            By Thomson's principle (Theorem~\ref{thm:thomson}), the following minimization program is solved by the electrical flow with unit resistance on every edge 
            \begin{equation}
            \label{eq:flow_min_var_optimization}
            \begin{aligned}
                \min_{f} \; &  \sum_{k, \ell } f_{k, n+\ell}^2 \\
            \text{s.t. } & \sum_{\ell} f_{i, n+\ell} = \sum_{k} f_{k, n+j} = 1 \\
                        & \sum_{\ell} f_{i', n+\ell} = 0, \quad \forall i' \in [n] \setminus \{i\} \\
                        & \sum_{k} f_{k, n+{j'}} = 0, \quad \forall j' \in [m] \setminus \{j\}.
            \end{aligned}
            \end{equation}
            Furthermore, the solution $\ii$ satisfies $\sum_{x,y} \ii^2_{x,y} = R(u_i, v_j)$.
        
            Applying Hoeffding's inequality, we have
            \begin{align*}
            \big| \hat{M}_{ij} - M^\ast_{ij} \big| \le  \sigma \sqrt{2 R(u_i, v_j) \log (2 nm / \delta)},
            \end{align*}
            with probability at least $1-\frac{\delta}{nm}$.
            Finally, we invoke
            the union bound over all entries 
            and obtain the desired result.
        \end{proof}
        
        Next, we state a local minimax lower bound for estimating entry $(i,j)$ under the framework of~\cite{cai2015framework}. This result implies that the electrical flow estimator is optimal for every instance of additive $M^*$, which is stronger than the usual worst case minimax lower bound. 

        \begin{thm}
        \label{thm:minimax_effective_resistance}
            For an arbitrary additive model $X^\ast$ and a fixed pair of $(i,j)$ that is connected, we have
            \begin{align}
            \label{eq:local_minimax}
                \sup_{Y^\ast} \; \inf_{\hat M} \; \max_{M^\ast \in \{X^\ast, Y^\ast\}} \; \mathbb{E} \left[ ( \hat M_{ij} - M_{ij}^\ast)^2 \right] \ge  \frac{2 \sigma^2 R(u_i, v_j)}{27},
            \end{align}
            where the supremum is over all additive models
            and the infimum is over all estimators $\hat M$.
        \end{thm}
        
        \begin{proof}
            We first construct an instance $Z^\ast$. 
            Suppose $X^\ast$ satisfies $X^\ast_{k \ell } = a_k + b_{\ell}$.
            Let $\vv$ be the corresponding voltage vector for the electrical flow. 
            We construct $Z^\ast$ as $Y^{\ast}_{k \ell} = \epsilon (\vv_k + a_k - \vv_{n+\ell} + b_{\ell})$, for some $\epsilon\in(0,1)$.
            It follows that $Z^\ast_{k \ell} - X^\ast_{k \ell} = \epsilon (\vv_k -\vv_{n+\ell})$. Accordingly, we have 
            $Z^\ast_{ij} - X^\ast_{ij} = \epsilon(\vv_i - \vv_j) = \epsilon R(u_i, v_j)$.
            Hence, the KL divergence between the distributions is
             \begin{align*}
                \KL (f_{X^\ast} \Vert f_{Z^\ast}) &= \frac{1}{2\sigma^2} \sum_{k, \ell} \Omega_{k \ell} (X_{k \ell}^\ast - Z^\ast_{k \ell }) 
                = \frac{\epsilon^2}{2\sigma^2}  \sum_{ x, y} (\vv_x - \vv_y)^2 
                = \frac{\epsilon^2 R(u_i, v_j)}{2\sigma^2}.
            \end{align*}
            To lower bound the LHS of~\eqref{eq:local_minimax}, we use the instance $Z^\ast$ constructed before
            \begin{align*}
                \sup_{Y^\ast}\; \inf_{\hat M} \; \max_{M^\ast \in \{X^\ast, Y^\ast\}} \; \mathbb{E} \left[ ( \hat M_{ij} - M_{ij}^\ast)^2 \right] 
                \ge \inf_{\hat M} \; \max_{M^\ast \in \{X^\ast, Z^\ast\}} \; \mathbb{E} \left[ ( \hat M_{ij} - M_{ij}^\ast)^2 \right].
            \end{align*}
            A standard application of Le Cam's method yields
            \begin{align*}
                \inf_{\hat M} \; \max_{M^\ast \in \{X^\ast, Z^\ast\}} \; \mathbb{E} \left[ ( \hat M_{ij} - M_{ij}^\ast)^2 \right] \ge \frac{\epsilon^2 R^2(u_i, v_j)}{8} \left(1 - \sqrt{\frac{\epsilon^2 R(u_i, v_j)}{4\sigma^2}} \right).
            \end{align*}
            Setting $\epsilon=\frac{4\sigma}{3\sqrt{R(u_i,v_j)}}$ gives us the desired result.
        \end{proof}
        
        Comparing Theorem~\ref{thm:flow_est_upper_bound} and Theorem~\ref{thm:minimax_effective_resistance}, we conclude that our estimator achieves the minimax error rate up to logarithmic factors. 
        Furthermore, our results reveal how the inherent challenge of estimating a specific entry hinges on the graph's connectivity. If two vertices are more connected, estimating the relevant entry becomes easier. 
        From an information propagation standpoint, our theory pins down one fact---the amount of information for estimating a particular entry is proportional to the resistance distance between vertices, reflecting their proximity in the graph.

\begin{remark}
To prove the minimax optimality in Theorem~\ref{thm:minimax_effective_resistance}, we explicitly constructed the hard instances using the voltage vector in the electrical network.
We can show that the voltage vector corresponds to the dual variables in the variance-minimizing program which optimizes over flows to wolve for the effective resistance. 
This proof technique can be extended to show minimax optimality of least squares for general linear regression problems, where the hard instance is constructed using the dual variables of the optimization task that minimizes the variance amongst linear estimators subject to unbiasedness constraints.
The use of dual variables to construct minimax lower bounds provides a straightforward and elegant proof of the minimax optimality of the least squares estimator. Details of this generalization are provided in 
Appendix~\ref{appen:linear_regression}.
\end{remark}

\begin{remark}
Recall that in Theorem~\ref{thm:equivalence_efe_lse}, we established the equivalence between the electrical flow estimator and the least squares estimator. In fact, Theorem~\ref{thm:minimax_effective_resistance} can be proved through this equivalence, \cyedit{as the least squares estimator is known to be asymptotically locally minimax optimal due to an equivalence to the maximum likelihood estimator (MLE) under Gaussian noise. This follows from a result of Hajek and Le Cam showing that the MLE is asymptotically locally minimax optimal for models with sufficient regularity \cite{hajek1970characterization,hajek1972local,le1972limits}}.
\cyreplace{Therefore, the minimax lower bound we derive is not entirely new but rather provides a different perspective and proof technique that extends beyond this model.}
{While the local minimax optimality of least squares has been previously known, we provide a new proof technique for the minimax lower bound via a simple and explicit construction using dual variables.} Moreover, our lower bound has physical interpretation in terms of effective resistance, whereas the lower bound for the general least squares estimator does not provide this insight.
\end{remark}    
                   
\subsubsection{Uniformly Minimum-Variance Unbiased Estimator~(UMVUE)}
\label{subsubsec:equivalence_umvue}
Under the strongest assumption of i.i.d.~Gaussian noise, it follows that the electrical estimator is the uniform minimum variance unbiased estimator (UMVUE). 
This property follows from the equivalence to the least squares estimator shown in Theorem~\ref{thm:equivalence_efe_lse}.

\begin{assum}
\label{assum:iid_Gaussian}
We assume $E_{ij}$'s are i.i.d.~$\normal(0,\sigma^2)$, for an unknown $\sigma^2>0$.
\end{assum}

\begin{thm}[UMVUE]
\label{thm:umvue}
Under Assumption~\ref{assum:iid_Gaussian}, the electrical flow estimator $\hat M_{ij}^{\efe}$ is the UMVUE of $M^\ast_{ij}$, for connected $(i,j)$. Its variance is
$\var \left[ \hat M_{ij}^{\efe}\right] = \sigma^2 R(u_i,v_j)$, for all connected $(i,j)$.
\end{thm}

The result follows from Theorem~3.7 from~\cite{shao2003mathematical}, stating that the least squares estimator is the UMVUE, which we restate here: \cycomment{Do we even need to state this Theorem and the proof? ...}
\begin{thm}[Theorem~3.7 from~\cite{shao2003mathematical}]
\label{thm:lr_umvue}
Let $X=\left(X_1, \ldots, X_N\right), \varepsilon=\left(\varepsilon_1, \ldots, \varepsilon_N\right)$. 
Observed data is denoted as $X=Z \phi^\ast 
+\varepsilon.$
Assume $\varepsilon$ is distributed as  $\normal \left(0, \sigma^2 I_N\right)$ with an unknown $\sigma^2>0$. The least squares estimator $w^\top \hat \phi$, where $\hat \phi = \underset{\phi}{\operatorname{argmin}} \norm{X - Z \phi}_2^2$, is the UMVUE for estimable $w^\top \phi^\ast$\footnote{If there exists an unbiased estimator of $w^\top\phi$, then $w^\top \beta$ is called an estimable parameter.}.
\end{thm}

\begin{proof}[Proof of Theorem \ref{thm:umvue}]
In our context, the parameters of interest are the latent factors $\aa^\ast$ and $\bb^\ast$, thus $\phi^\ast = (a_1^\ast, a_2^\ast, \dots, a_n^\ast, b_1^\ast, \dots, b_m^\ast) \in \R^{n+m}$. 
Each observation of the vector $X$ corresponds to an observed matrix entry: $X = \Vec_{\Omega}(M) \in \R^{n_e}$. When $X_s$ corresponds to $M_{k \ell } = a_k^\ast + b_{\ell}^\ast + E_{k \ell}$, the covariate $Z_s$ is $e_{k} + e_{n+\ell}$, where $e_k$ and $e_{n+\ell}$ are the standard basis vector in $\R^{n+m}$. 
By applying Theorem~\ref{thm:lr_umvue} with $w = e_i + e_{n+j}$, it follows that the least squares estimator $\hat M_{ij}^{\lse}$ is the UMVUE for $M_{ij}^\ast$. The desired result follows from the equivalence between EFE and LSE stated in Theorem~\ref{thm:equivalence_efe_lse}.
The variance is implied by Theorem~\ref{thm:flow_est_upper_bound}.
\end{proof}

By Thomson's Principle~(Theorem~\ref{thm:thomson}), we know that the electrical flow estimator achieves minimum variance out of all the unit flow estimators. Theorem~\ref{thm:umvue} shows something stronger---the electrical flow estimator has the minimum variance out of all the unbiased estimators, uniformly for all possible values of the signal. 
The connection between the electrical flow estimator and the least squares estimator offers a two-way street. 
On the one hand, we obtain a deeper understanding of the minimum achievable variance of unbiased estimators, which is tied to the effective resistance measuring graph connectivity. 
On the other hand, we demonstrate that the electrical flow estimator is UMVUE, which would have been a daunting task had we not known the equivalence in Theorem~\ref{thm:equivalence_efe_lse}.

\subsection{Application to Panel Data with Two-Way Fixed Effects}
\label{subsec:panel_data}

Consider a panel data setting in which we observe outcomes of units over time, and units can be either exposed to treatment or control at any given time. Two-Way Fixed Effect (TWFE) regressions are commonly used to analyze panel data for the purpose of estimating an average treatment effect. It relies on the Two-Way Fixed Effect model, which posits that
\[Y_{it} =  \alpha_i + \gamma_t + \beta_{it} X_{it} + E_{it}, \quad\text{ for } i \in[N], \; t\in[T],\]
where $Y_{it}$ denotes the outcome, $X_{it}$ denotes the treatment indicator, $E_{it}$ denotes an unobserved noise, and $\alpha_i, \gamma_t, \beta_{it}$ are unknown parameters of the model.
The typical TWFE model in the literature additionally assumes homogenous treatment effect, i.e. $\beta_{it} = \beta$, which can be unrealistic in practice. TWFE regression simply estimates the unknown parameters $\{\alpha_i\}_{i\in[n]}, \{\gamma_t\}_{t\in[T]}$, and $\beta$ via least squares regression on the observations.
While TWFE regressions are well understood for simple treatment patterns, such as staggered adoption, there is still a gap in our understanding when it comes to arbitrary treatment patterns.

We apply our new class of flow estimators to estimate individual causal effects for panel data under the heterogeneous TWFE model as defined in~\eqref{eq:linear_fixed_effect}, where we assume that $\beta_{it} = a_i + b_t$. 
For notation simplicity, we define $F^\ast \in \R^{N \times T}$ and $G^\ast \in \R^{N \times T}$ as the expectations of the outcomes in the control group and the treatment group, respectively, i.e.~$F^\ast_{it} = \alpha_i + \gamma_t$ and $G^\ast_{it} = (\alpha_i + a_i) + (\gamma_t + b_t)$ for all $(i,t)$. Therefore, we have the following relation
    \begin{align*}
        Y_{it} = 
        \begin{cases}
            F^\ast_{it} + E_{it} & \Omega^0_{it} = 1 \\
            G^\ast_{it} + E_{it} & \Omega^1_{it} = 1.
        \end{cases}
    \end{align*}
    The causal effects are the differences between $F^\ast$ and $G^\ast$
   $\beta_{it} = G^\ast_{it} -  F^\ast_{it}$. The individual treatment effect $\beta_{it}$ can be estimated by subtracting estimates obtained from the treatment graph $\mathcal{G}(\Omega^1)$ and the control graph $\mathcal{G}(\Omega^0)$.
   For any {\em arbitrary treatment pattern}, under the heterogeneous two way additive fixed effects model, the unbiased estimate given by our network flow estimator for the outcome of a unit-time pair $(i,t)$ under treatment or control is constructed by weighting observations according to any unit $(i,t)$ flow in the graphs associated with the treatment and control observations $\Omega^1$ and $\Omega^0$ respectively. The network flow estimator for $\beta_{it}$ requires that $i$ and $t$ are connected in both the treatment graph $\mathcal{G}(\Omega^1)$ and the control graph $\mathcal{G}(\Omega^0)$.

The electrical flow estimator, denoted $\hat \beta_{it}^{\efe}  = \hat{G}_{it}^{\efe} - \hat{F}_{it}^{\efe}$, results from weighting observations according to the $(i,t)$-electrical flow in the graphs associated with the treatment and control observations, respectively.
    Due to the properties given in Section \ref{sec:guarantees}, we know that $\hat \beta_{it}^{\efe}$ is minimax optimal and variance minimizing.  
We obtain the following error bound by applying Theorem~\ref{thm:flow_est_upper_bound} to $ \hat{G}_{it}^{\efe}$ and $\hat{F}_{it}^{\efe}$. Let $R_{\Omega^0} (u_i, v_j)$ denote the effective resistance of $u_i$ and $v_j$ in graph $\mathcal{G}(\Omega^0)$ and let $R_{\Omega^1} (u_i, v_j)$ denote the effective resistance of $u_i$ and $v_j$ in graph $\mathcal{G}(\Omega^1)$. 
    
    \begin{thm}
        \label{thm:fixed_effects_upper_bound}
        Under Assumptions~\ref{assum:zero-mean_noise} and~\ref{assum:indepenent_subG}, there exists an absolute constant $C>0$ such that with probability at least $1-\delta$, we have
        \begin{align}
            \big( \hat{\beta}_{it}^{\efe} - \beta_{it} \big)^2 \le C \sigma^2 \left[ R_{\Omega^0} (u_i, v_t) + R_{\Omega^1} (u_i, v_t) \right] \log ( NT / \delta), 
        \end{align}
        for all $(i,t)$ such that $R_{\Omega^0} (u_i, v_t) +  R_{\Omega^1} (u_i, v_t) < \infty$.
    \end{thm}


The estimation quality of a causal effect $\beta_{it}$ is related to the connectivity between $u_i$ and $v_j$ in the induced bipartite graphs for the control and the treatment groups as characterized by the effective resistance, which gives a fine-grained and intuitive understanding of the impact of the given treatment pattern on estimation.
    Theorem \ref{thm:unidentifiability} implies that if $u_i$ and $v_j$ are not connected in either of the graphs, i.e. either $R_{\Omega^0} (u_i, v_j)=\infty$ or $R_{\Omega^1} (u_i, v_j)=\infty$, then $\beta_{it}$ is unidentifiable. 
        Individual causal effects $\beta_{it}$ that can be estimated very well correspond to pairs $u_i$ and $v_t$ with low effective resistance, implying that they are well connected in both graphs $\cG(\Omega^0)$ and $\cG(\Omega^1)$.  
    Our algorithm and guarantees hold for any given set of observations $\Omega^0$ and $\Omega^1$, in contrast to the very specific and strong assumptions that are placed on the observation patterns in the existing literature on panel data with two-way fixed effects.

\subsubsection{Connection to Difference in Differences Estimator}


Our approach introduces a family of estimators parameterized by flows over networks, where flows restricted to length-3 paths correspond to the difference-in-differences (DiD) estimator. 
    Consider the instance shown in Figure~\ref{fig:DID_equivalence}, where $X_{i2} = 1$, and thus we only have a direct observation of the outcome of unit $i$ at time $2$ under treatment but not under control. However, suppose that in the bipartite graph $\cG(\Omega^0)$ associated with the observations assigned to control, $u_i$ and $v_2$ are connected via a path $u_i \to v_{1} \to u_{j} \to v_2$, which means that $X_{i1} = 0$, $X_{j1} = 0$, and $X_{j2} = 0$. Even though customer $i$ is exposed to the promotion at time $2$, we can then use the outcomes associated with $Y_{i1}, Y_{j1}$, and $Y_{j2}$ to estimate the potential outcome of customer $i$ at time $2$ under control.
    This results in the single path estimator of $\hat \beta_{i2} = Y_{i2} - (Y_{j2} + Y_{i1} - Y_{j1}) = (Y_{i2} - Y_{j2}) - (Y_{i1} - Y_{j1})$. 
    
    \begin{figure}[h]
        \centering
\includegraphics[width=0.8\linewidth]{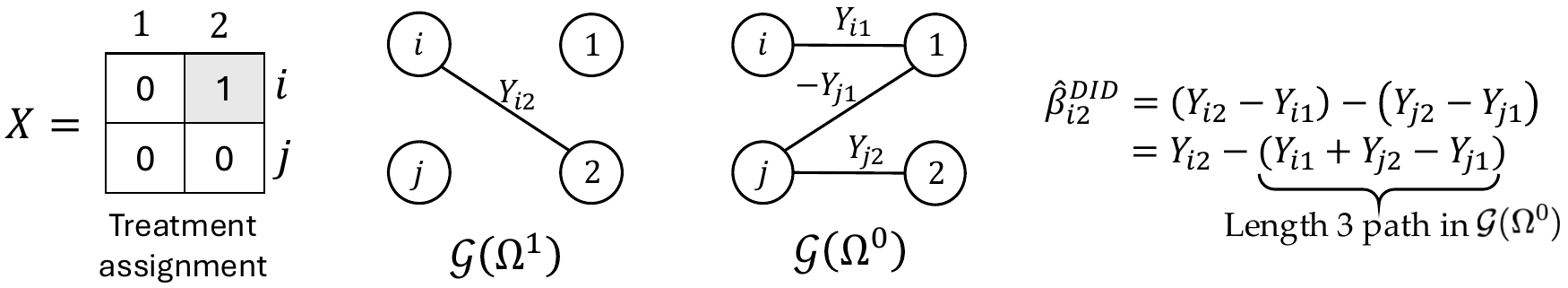}
        \caption{Equivalence of DiD estimators to flow estimators constrained to length 3 paths.}
        \label{fig:DID_equivalence}
    \end{figure}
    
    In this specific example, our single path estimator recovers the \emph{difference-in-differences} estimator, which estimates the treatment's causal effect by subtracting the difference in outcomes between the treated unit $i$ and the control unit $j$ at time $1$ before the treatment, from the difference in outcomes at time $2$ after the treatment. Our algorithmic framework can thus be interpreted as an extension of \emph{difference-in-differences} estimation, where the difference-in-differences method is a special case of our network flow estimator with length 3 paths.
%
    The general class of flow estimators offers more flexibility when the treatment and observation pattern may be unstructured. When there does not exist connecting paths of length 3 or shorter, the DiD estimator is unable to provide any estimate, while the flow estimator shows that the heterogeneous treatment effect $\beta_{it}$ is identifiable as long as $i$ and $t$ are connected in both the treatment and control graphs.
           

\subsubsection{Connection to TWFE Regression}
    
    The equivalence between EFE and LSE in Theorem \ref{thm:equivalence_efe_lse} implies that the EFE is equivalent to a variation of TWFE regression that accounts for heterogeneous treatment effects (satisfying additive fixed effects),
    \begin{align}
    (\hat \mu, \hat \nu) = \operatorname{argmin}_{\mu, \nu, \alpha, \gamma} \; \textstyle\sum_{i, t} \Omega^0_{it} (\alpha_i + \gamma_t - Y_{it} )^2 + \textstyle\sum_{i, t} \Omega^1_{it} (\alpha_i + \gamma_t + \mu_i + \nu_t -Y_{it})^2.
    \label{eq:TWFE_regression}
    \end{align}
    Adding the estimates yields $\hat \beta^{\lse}_{it} = \hat \mu_i + \hat \nu_t$.
    Since LSE is equivalent to EFE, we have $\hat \beta^{\efe}_{it} = \hat \beta^{\lse}_{it}$ and \eqref{eq:TWFE_regression} can be solved by using the electrical flows from the treatment graph and the control graph. 
    We can thus interpret TWFE regression estimates as a sophisticated method for combining paths--potentially longer and overlapping ones--not just those of length-3 as in DiD. 
    Because of this flexibility, EFE (and LSE) can recover treatment effects for arbitrary treatment patterns, even when DiD fails due to missing length-3 paths. This offers a clean interpretation of the minimum variance by leveraging the notion of effective resistance in electrical networks. In an electrical network, effective resistance between two vertices is lower if there are multiple short paths connecting them, indicating greater connectivity. \cite{jochmans2019fixed} has previously showed that graph connectivity as characterized by the spectral gap and overlap of direct neighbors is key to estimation accuracy of TWFE regression for fixed effect models. Our result further gives an explicit and exact equivalence of the variance of the best estimator with the effective resistance, providing fine-grained entry-specific results.

    \subsubsection{Example with a Staggered Exposure Model}
        
        We discuss a concrete observation pattern that demonstrates the efficacy of our method.  For simplicity, assume $N=T$ and the dataset is balanced.
        Suppose we partition the units and the time periods into $G$ groups. We assume that groups are of the same size, $H = N / G$. Let the unit groups be denoted as $R_1, R_2 ,\dots, R_G$ and the time groups be denoted as $C_1, C_2 ,\dots, C_G$. If a unit $i \in R_{g}$ for some $g \in[G]$, then $X_{it}=1$ if and only if $t \in C_{g}$ or $t \in C_{g+1}$. 
        We show a plot of this pattern in Figure~\ref{fig:FE_example}. 
        This suggests a staggered exposure setting in which groups get exposed to treatment at staggered times, but instead of the treatment persisting until the end of the experiment (as in a staggered adoption setting), each unit is only exposed to the treatment for a fixed length of time. 

        \begin{figure}[h]
             \centering
             \includegraphics[width=0.3\textwidth]{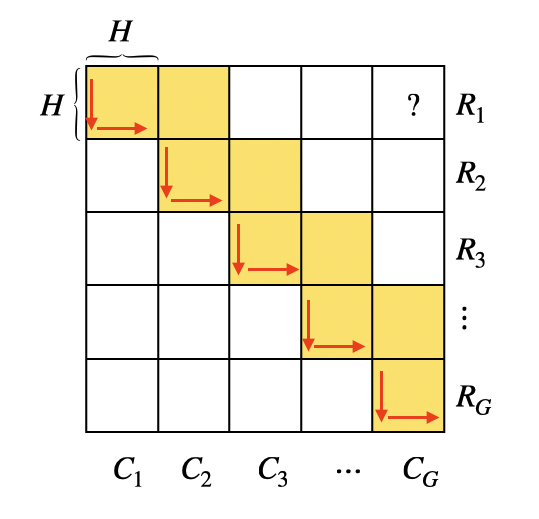}
             \caption{Treatment pattern corresponding to a staggered exposure model, where units are partitioned into $G$ groups and each group is exposed to the treatment for a fixed length of time beginning at staggered times. The red arrows show an example path from row $1$ to column $T$.}
             \label{fig:FE_example}
        \end{figure}
        
        Consider estimating the causal effect of unit $1$ at time $T$. Estimating control outcomes is relatively easy since many of the entries nearby $(1,T)$ are observed under control. Estimating treatment outcomes is more difficult, as many classical matrix completion methods require a densely observed block of a sufficiently large size under treatment. One can show that there does not exist any permutation of the rows and columns such that there is a densely observed block under treatment. This is due to the fact that the treatment pattern resembles a block identity matrix, and if $H$ is small, the graph $\cG(\Omega^1)$ is structured like a thin ring graph, which does not contain large cliques. However, we can verify that $u_1$ and $v_T$ are connected in graph $\cG(\Omega^1)$ by $H$ disjoint paths of length $2G$. 
        In Figure~\ref{fig:FE_example} we show one such example path in red arrows, consisting of entries $(1,1), (H,1), (H,H), (2H, H), (2H, 2H), \dots, (GH, GH)$. \cyedit{Recall that the effective resistance is the minimizer of the program defined in \eqref{eq:flow_min_var_optimization}, and thus the energy of any flow is an upper bound on the effective resistance. Therefore, we can consider the flow which puts $1/H$ flow on each of the disjoint paths of length $2G$, resulting in an upper bound on the effective resistance in the treatment graph $R_{\Omega^1}(u_1,v_T) \leq 2G/H = 2G^2/N$. The control graph is very dense, especially when $H$ is small relative to $N$. In particular there is a upper right submatrix of size $mH \times mH$ for $m = \lfloor \frac{N-H}{2H} \rfloor$, which contains $(1,T)$ and forms a bipartite clique in $\mathcal{G}(\Omega^0)$. Thus there are $mH-1$ disjoint paths of length 3 connecting $u_1$ and $v_T$, and one path of length 1 directly connecting $u_1$ and $v_T$ implying an upper bound on the effective resistance of $R_{\Omega^0}(u_1,v_T) \leq 3/mH = 6/(N-H)$.}
        It follows from Theorem~\ref{thm:fixed_effects_upper_bound} that with high probability
        \begin{align*}
            |\hat \beta_{1T} - \beta_{1T}| \lesssim \sigma \sqrt{\frac{G^2}{N}},
        \end{align*}
        ignoring logarithmic factors. When the number of groups satisfies $G = o(\sqrt{N})$, it is guaranteed that $\hat \beta_{1T}$ is a consistent estimator. In this example, the difference-in-differences estimator is unable to produce an estimate as there is no length 3 path in the bipartite graph associated with the outcomes of the treated units. The estimators in \cite{xi2023nonuniform, agarwal2021causal} are also unable to produce an estimate as there is no dense submatrix containing entry $(1,T)$. Classical estimators from the matrix completion literature also fail to provide guarantees due to the highly non-uniform observation pattern.

\subsubsection{Synthetic Simulations}
We conduct synthetic experiments with a sparsified staircase treatment pattern as depicted in Figure~\ref{fig:staircase_treatment_matrix}. The treatment graph $\mathcal{G}(\Omega^1)$ as depicted in Figure~\ref{fig:staircase_treatment_graph} has a zig-zag pattern. For most entries $(i,j)$ there does not exist a length-3 path in the graph $\mathcal{G}(\Omega^1)$, and as a result DiD fails in this scenario. The data matrix Figure~\ref{fig:simulation_eff_res} shows a heatmap of the sum of the effective resistances for the graphs corresponding to the treatment and control observations.
We generate row and column fixed effects from a Gaussian distribution and set the variance of the observation noise as $\sigma^2 = 0.01$. For a given realization of the observations we compute the corresponding electrical flow estimates $\hat{\beta}_{it}^{\efe}$ for all pairs $(i,t)$. We repeat the experiment 1000 times and compute the entrywise mean squared error for each entry.
Figure~\ref{fig:ratio} shows a histogram of the ratio between the entrywise mean squared error of the EFE and the sum of the effective resistances on the treatment and control graphs. Figure~\ref{fig:simulation_MSE} shows a heatmap of the entrywise mean squared error, which visually matches the heatmap of the effective resistances as depicted in Figure~\ref{fig:simulation_eff_res}. These empirical results confirm our theoretical results as stated in Theorem \ref{thm:flow_est_upper_bound} that the estimation error is proportional to the sum of the effective resistance from the treatment and control graphs, where the multiplicative factor is determined by the variance of the observation noise $\sigma^2 = 0.01$. 

\begin{figure}
    \centering
    \begin{subfigure}[b]{0.28\textwidth}
    \includegraphics[width=\textwidth]{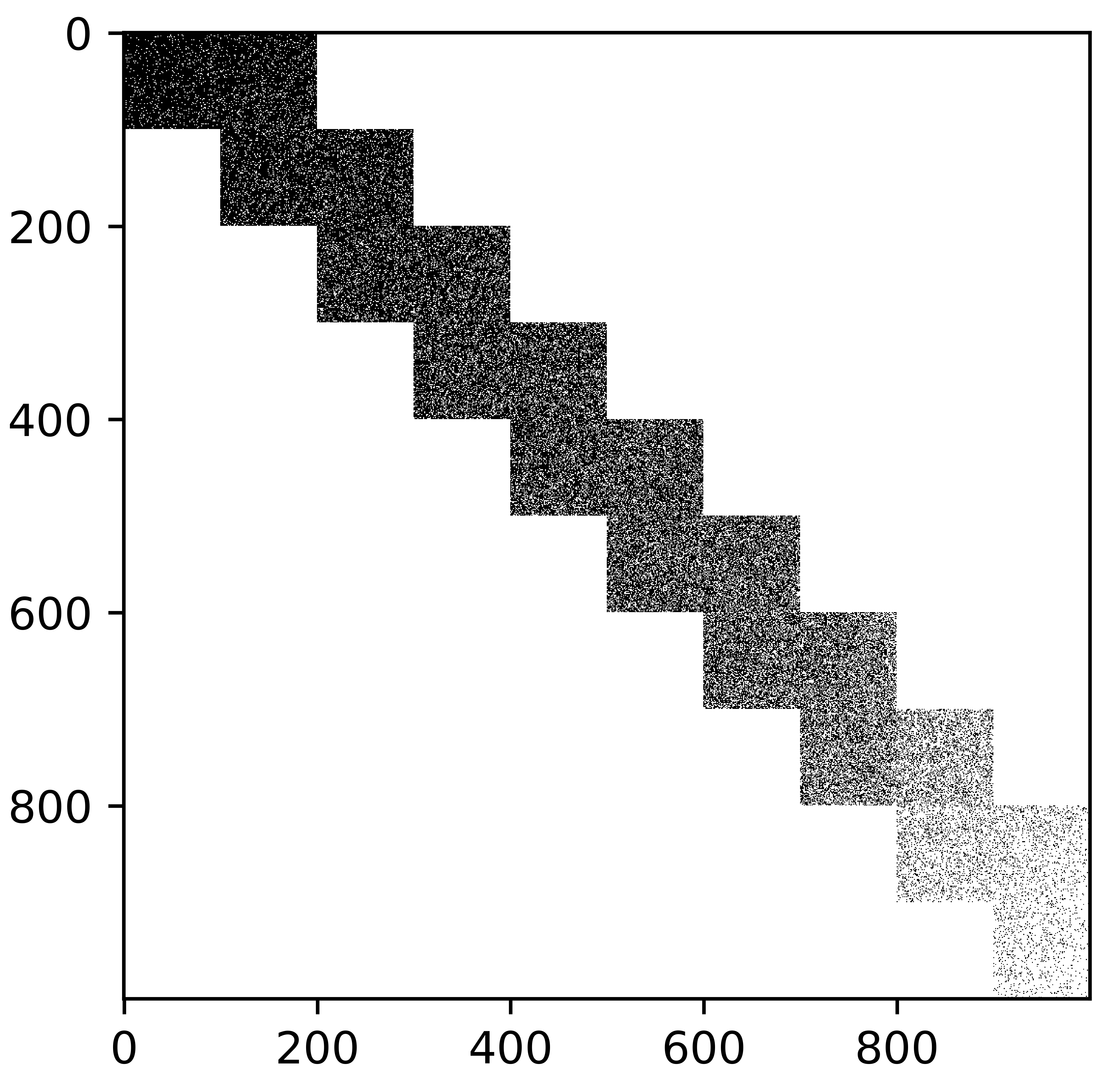}
    \caption{~}
    \label{fig:staircase_treatment_matrix}
    \end{subfigure}
    \quad
    \begin{subfigure}[b]{0.13\textwidth}
    \includegraphics[width=\textwidth]{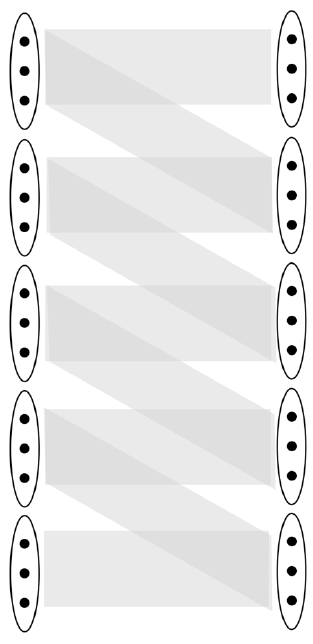}
    \caption{~}
    \label{fig:staircase_treatment_graph}
    \end{subfigure}
    \quad
    \begin{subfigure}[b]{0.5\textwidth}
    \includegraphics[width=\textwidth]{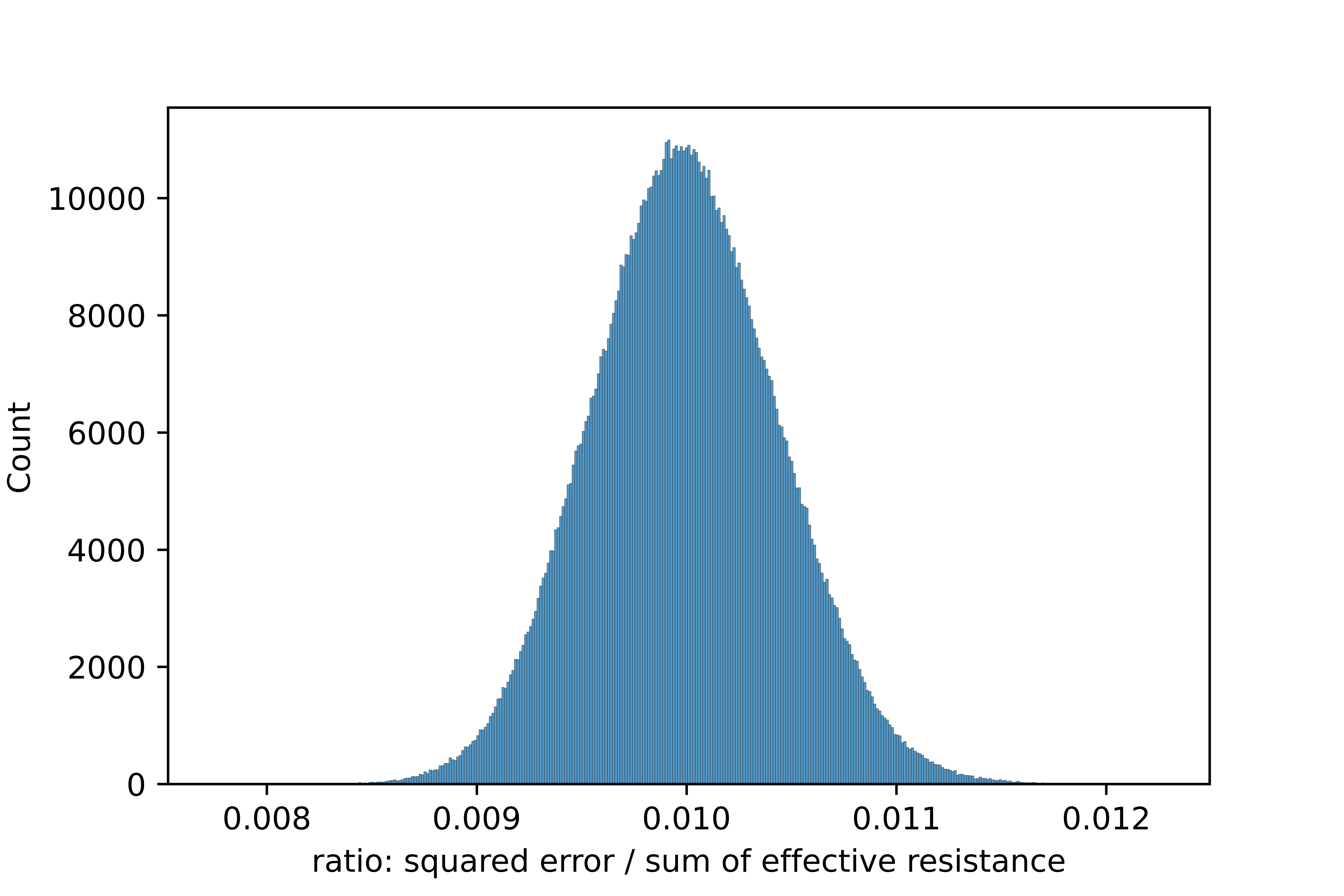}
    \caption{~} \label{fig:ratio}
    \end{subfigure}
    \caption{(a) The treatment pattern $X$ has a staircase structure which gets progressively sparser to the bottom right. (b) The corresponding treatment graph $\mathcal{G}(\Omega^1)$ has a zig-zag pattern. (b) Histogram of the ratio between the entrywise squared error and the sum of effective resistances. As predicted by our theory, the ratio is distributed as a normal distribution with mean equal to $\sigma^2 = 0.01$.}
\end{figure}

\begin{figure}
    \centering
    \begin{subfigure}[b]{0.35\textwidth}
    \includegraphics[width=\textwidth]{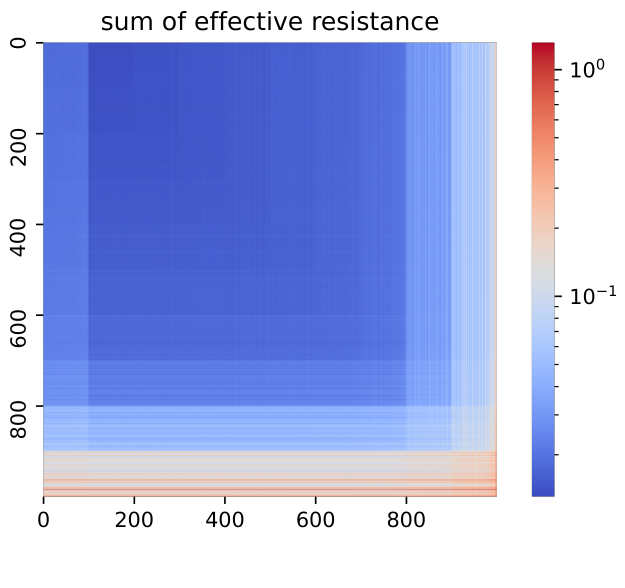}
    \caption{~} \label{fig:simulation_eff_res}
    \end{subfigure}
    \begin{subfigure}[b]{0.35\textwidth}
    \includegraphics[width=\textwidth]{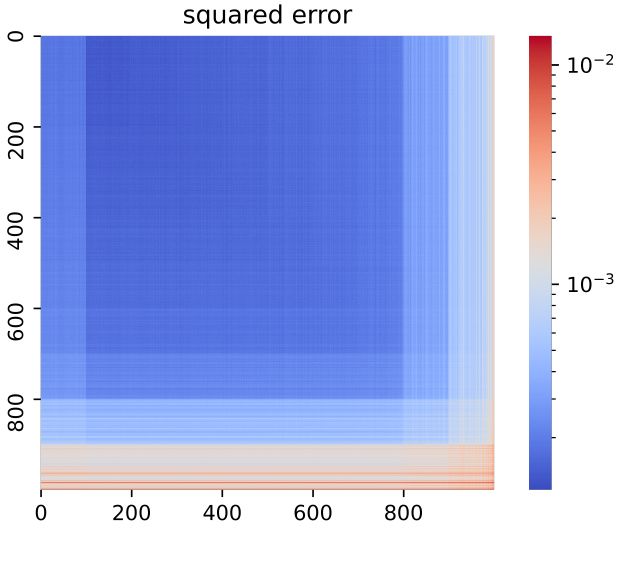}
    \caption{~} \label{fig:simulation_MSE}
    \end{subfigure}
    \caption{We can visually verify that the entrywise squared estimation error is proportional to the effective resistances. (a) Heatmap of the sum of effective resistances in the graphs corresponding to treated and control observations. (b) Heatmap of the entrywise squared estimation error.}
    \label{fig:simulation}
\end{figure}

\section{Entry-Specific Estimation for Rank-1 Matrices}

In this section, we shift our focus to estimating rank-1 matrices, introduced in Definition~\ref{defn:rank1}.
We propose a novel construction of a path-induced estimator, which bears resemblance to the difference-on-path estimator~\eqref{eq:single_path_est} for estimating additive matrices. After introducing our estimator, we present entry-specific error upper bounds that depend on the maximum number of disjoint paths and the maximum length of these paths. We show that the upper bounds match the minimax lower bound when the observation pattern is sufficiently dense or contain dense submatrices. \cycomment{what do we mean by "has a particular structure", I deleted that phrase since I wasn't sure what it was trying to refer to.} 

\subsection{Ratio-on-Path Estimator}
\label{subsec:ratio_on_path}

    Recall that for a fixed $\Omega$, we construct an undirected bipartite graph $\cG(\Omega) = (\cV, \cE({\Omega}) )$ where the vertex sets represent the rows and columns, and the edges are determined by the sparsity pattern of $\Omega$. See Subsection~\ref{subsec:additive_graph_construction} for details and notation.  
    Consider estimating $M_{ij}^\ast$ for a fixed index $(i, j) \in [n] \times [m]$. Let $\path(\xx, \ell)$ denote a simple path in $\cG(\Omega)$ connecting $u_{i}$ and $v_{j}$: $u_{x_{1}} (= u_{i}) \to v_{x_2} \to u_{x_3} \to v_{x_4} \to \dots \to u_{x_{\ell}} \to v_{x_{\ell+1}} (=v_{j})$. 
    We can construct an estimator from this path according to
    \begin{align}
    \label{eq:single_path_est_rank1}
        \hat{M} _{ij}^{\path(\boldsymbol{x}, \ell)} = \frac{ \alpha^{\path(\boldsymbol{x}, \ell)} }{\beta^{\path(\xx, \ell)}} = \frac{  \prod_{s=1, 3, \dots, \ell} Y_{x_{s} x_{s+1}} }{ \prod_{s=3,5,\dots,\ell} Y_{x_{s} x_{s-1}} }.
    \end{align}
    In the numerator $\alpha^{\path(\xx, \ell)}$, we multiply all observations appearing on an edge crossing from a left node to a right one. 
    In the denominator $\beta^{\path(\xx, \ell)}$, we multiply all observations appearing on an edge crossing from a right node to a left one. The intuition behind the design of~\eqref{eq:single_path_est_rank1} is that the factors irrelevant to estimating $M_{ij}^\ast = a_i^\ast b_j^\ast$ can cancel each other out from the terms on the numerator and the denominator, leaving a noisy perturbation of $a_i^\ast b_j^\ast$.  However, without assuming any additional properties of $M^\ast$, it is likely that \eqref{eq:single_path_est_rank1} is not unbiased due to the noise in the denominator. \cyedit{Furthermore, the denominator of~\eqref{eq:single_path_est_rank1} may not be bounded away from zero due to the noise in the observation, even under the assumption that the entries of $M^*$ are bounded away from zero.}

    \cycomment{I went ahead and removed the other estimator in between and skipped right to the final estimator.}
    To address these challenges, we propose the following estimator: We consider a collection of (edge-)disjoint paths $\{ \path(\xx^k, \ell_k) \}_{k \in [K]}$ connecting $u_i$ and $v_j$. Let $K \equiv K(i,j)$ be the maximum number of disjoint paths from $u_i$ to $v_j$. We compute a ratio of the averages of $\alpha^{\path(\boldsymbol{x}, \ell)} \cdot \beta^{\path(\xx, \ell)}$ and $\left(\beta^{\path(\xx, \ell)}\right)^2$ from~\eqref{eq:single_path_est_rank1}, where \cyedit{we have multiplied the numerator and denominator by $\beta^{\path(\xx, \ell)}$ to avoid instability when $\beta^{\path(\xx, \ell)}$ may be close to zero}.
    \begin{equation}
    \label{eq:multi_path_est_stable}
        \hat M_{ij} = \frac{ \frac1K \sum_{k} \alpha^{\path(\xx^k, \ell_k )} \cdot \beta^{\path(\xx^k, \ell_k )} }{  \frac1K \sum_{k} (\beta^{\path(\xx^k, \ell_k )})^2 }.
    \end{equation}
    The squared terms on the denominator make sure that with high probability, $\hat M_{ij}$ is well-defined for sufficiently large $K$, \cyedit{i.e. we just need sufficiently many paths for which $\beta^{\path(\xx, \ell)}$ is not too close to zero.}
    The expectation of the numerator of~\eqref{eq:multi_path_est_stable} is
    \begin{align*}
        \frac1K \sum_k \Bigg( \prod_{s=1,3,\dots, \ell_k} M^\ast_{x_s x_{s+1}} \Bigg) \cdot \Bigg( \prod_{s=3,5,\dots,\ell_k} M^\ast_{x_s x_{s-1}} \Bigg) &= \frac{a_i^\ast b_j^\ast}{K} \sum_k  \prod_{s=3,5,\dots,\ell_k}  (a_{x_s}^\ast b_{x_{s-1}}^\ast)^2.
    \end{align*}
    The expectation of the denominator of~\eqref{eq:multi_path_est_stable} is
    \begin{align*}
        \frac1K \sum_k \prod_{s=3,5,\dots,\ell_k} (M^\ast_{x_s x_{s-1} })^2 = \frac1K \sum_k  \prod_{s=3,5,\dots,\ell_k}  (a_{x_s}^\ast b_{x_{s-1}}^\ast)^2.
    \end{align*}
    \cyedit{Thus, since the $K$ paths are disjoint, the numerator and denominator are each averages of independent statistics, such that the numerator and denominator will each converge to its expectation for large $K$. In Theorem~\ref{thm:multi_path_est_upper_bound} we formally show that the estimator~\eqref{eq:multi_path_est_stable} is consistent as long as there are sufficiently many disjoint paths of sufficiently short length.}

    \RestyleAlgo{ruled}
    \begin{algorithm}
        \SetAlgoLined
        \SetKwFunction{algo}{PathEstimator}{}
        \SetKwProg{myalg}{Function}{}{}
        \myalg{\algo{$M, \Omega$}}{
        \For{$(i,j) \in [n] \times [m]$}{
            Find disjoint paths $\{\path(\xx^k, \ell_k)\}_{k\in[K]} \gets \texttt{PathFinder}(\cG_\Omega, i, j)$. \\
            \uIf{$K=0$}{
                $\hat M_{ij} \gets \Box$.
            }
            \Else{
                \For{$k \in [K]$}{
                    Compute $\alpha^{\path(\boldsymbol{x}^k, \ell_k)}$ and $\beta^{\path(\boldsymbol{x}^k, \ell_k)}$ according to~\eqref{eq:single_path_est_rank1}.
                }
                 Compute $\hat M_{ij}$ according to~\eqref{eq:multi_path_est_stable}.\\
            }
         }
         \KwRet{$\hat{M}$}
        }
    \caption{Path-induced estimation for additive model. \label{algo:path_est}}
    \end{algorithm}

    In line~3, the algorithm \texttt{PathFinder} outputs a set of disjoint paths with maximum size in $\cG(\Omega)$ between $u_i$ and $v_j$. 
    It is a standard routine that can be implemented using a maximum-flow algorithm; for completeness, we present one such implementation as Algorithm~\ref{algo:path_finder} in Appendix~\ref{appen:path_finder}. 
    According to~\cite{kleinberg2006algorithm}, the time complexity of such a path-finding algorithm is $O\left( (n+m) n_e \right)$, which is polynomial in $n$ and $m$. 
    Hence, the overall time complexity of Algorithm~\ref{algo:path_est} is $O\left( nm (n+m) n_e \right)$.

\subsection{Theoretical Guarantees}

    We obtain the following error bound that shows entry $(i,j)$ is easier to estimate if there are more paths between $u_i$ and $v_j$. We further assume that $\min\{ \min_i |a_i^\ast|, \min_{j} |b_j^\ast| \} \ge 1,$ which ensures that the latent factors are bounded away from zero. This condition is necessary to guarantee that the denominator in~\eqref{eq:multi_path_est_stable} is well bounded away from zero. We can replace the lower bound $1$ with an arbitrary positive constant.
    
    \begin{thm}
    \label{thm:multi_path_est_upper_bound}
        There exists an absolute constant $C>0$ such that the following is true. For $(i,j)$ with $K(i,j)\ge C$, with probability at least $1 - \delta$, we have
        \begin{align}
            \big| \hat M_{ij} - M_{ij}^\ast \big| \le C \sigma^L (1 + \inftynorm{M^\ast}^L)  \sqrt{\frac{ 2^L \log^{L+1}{(nm/\delta) }}{K}},
        \end{align}
        where $\inftynorm{M^\ast} = \max_{ij} |M_{ij}^\ast|$,
        $K$ is the maximum number of disjoint paths connecting $u_i$ and $v_j$, and $L = \max_{k} \ell_k$ is the maximum length of these paths.
    \end{thm}
    Note that the error upper bound depends inversely on the maximum number of disjoint paths $K$. This implies that as the number of disjoint paths between $u_i$ and $v_j$ increases, more information is encoded in the dataset for estimating entry $(i,j)$, thereby improving the estimation accuracy.
    However, we notice an exponential dependence on the maximum length $L$, indicating that the observation error accumulates exponentially as we estimate along a given path. In certain scenarios with multiple longer paths, our estimator may not guarantee optimal estimation as it could achieve better accuracy by disregarding noisy estimations based on longer paths. Simply combining disjoint paths through direct averaging is not always optimal and more work is needed to determine a better approach. Nevertheless, when the observation pattern is sufficiently dense or contains a dense submatrix, shorter paths are often available, allowing us to treat $L$ as a constant and achieve minimax optimal estimation. We prove Theorem~\ref{thm:multi_path_est_upper_bound} before introducing the optimality result.
    
    \begin{proof}[Proof of Theorem~\ref{thm:multi_path_est_upper_bound}]
        For ease of presentation, we define the following notation for all $k \in [K]$,
        \begin{align*}
            g_k(Y) &\coloneqq \bigg( \prod_{s=1, 3, \dots, \ell_k} Y_{x_{s}^k x_{s+1}^k } \bigg) \cdot \bigg( \prod_{s=3, 5, \dots, \ell_k} Y_{x_{s}^k x_{s-1}^k }  \bigg) \\
            h_k(Y) &\coloneqq \prod_{s=3, 5, \dots, \ell_k} Y_{x_{s}^k x_{s-1}^k }^2.
        \end{align*}
        As a polynomial of $Y$, the degree of each $g_k(Y)$ is $\ell_k$ and the degree of $h_k(Y)$ is $\ell_k - 1$. We first present an intermediate result showing the concentration of the numerator and the denominator of~\eqref{eq:multi_path_est_stable}, the proof of which is deferred to Appendix \ref{sec:appendix_proofs}.
        \begin{lem}
            \label{lem:concentration_num_denom}
            There exists an absolute constant $C>0$ such that with probability at least $1 - 3\delta$, we have
            \begin{align*}
                \bigg| \frac 1K \sum_{k\in[K]} g_k(Y) - \frac 1K \sum_{k\in[K]}  g_k(M^\ast) \bigg| &\le C \sigma^L (1 + \inftynorm{M^\ast}^L)  \sqrt{\frac{ 2^L \log^{L+1}{(nm/\delta) }}{K}} \\
                \bigg| \frac1K \sum_{k\in[K]}  h_k(Y) - \frac1K \sum_{k\in[K]}  h_k(M^\ast) \bigg| &\le C \sigma^L (1 + \inftynorm{M^\ast}^L)  \sqrt{\frac{ 2^L \log^{L+1}{(nm/\delta) }}{K}}.
            \end{align*}
        \end{lem}
        Applying Lemma~\ref{lem:concentration_num_denom}, we get
        \begin{align*}
            \frac1K \sum_{k\in[K]} \alpha^{\path(\xx^k, \ell_k)} \cdot \beta^{\path(\xx^k, \ell_k)} = \frac 1K \sum_{k\in[K]}  f_k(Y)  &=  \underbrace{\frac1K \sum_{k\in[K]} f_k(M^\ast)}_{\gamma_1} + \epsilon_1 \\
             \frac1K\sum_{k\in[K]} (\beta^{\path(\xx^k, \ell_k)})^2 = \frac1K \sum_{k\in[K]} h_k(Y) &= \underbrace{\frac1K\sum_{k\in[K]} h_k(M^\ast)}_{\gamma_2} + \epsilon_2.
        \end{align*}
        where $| \epsilon_1 |, | \epsilon_2|  \le C \sigma^L (1 + \inftynorm{M^\ast}^L)  \sqrt{\frac{ 2^L \log^{L+1}{(nm/\delta) }}{K}}$ with high probability. 
        Since each $M^\ast_{ij}$ is lower bounded by $1$, we derive that $\gamma_1\ge 1$ and $\gamma_2 \ge 1.$
        Hence, the estimation error of~\eqref{eq:multi_path_est_stable} can be bounded as
        \begin{align*}
            \big| \hat M_{ij}  - M_{ij}^\ast \big| &= \Bigg| \frac{\gamma_1 + \epsilon_1}{\gamma_2 + \epsilon_2} -  \frac{\gamma_1}{\gamma_2} \Bigg|  
            = \Bigg|  \frac{ \gamma_2 \epsilon_1 - \gamma_1 \epsilon_2 }{\gamma_2 (\gamma_2+\epsilon_2)} \Bigg| \lesssim  \sigma^L (1 + \inftynorm{M^\ast}^L)  \sqrt{\frac{ 2^L \log^{L+1}{(nm/\delta) }}{K}},
        \end{align*}
        when $K$ is sufficiently large.
    \end{proof}
    
    Next we state a minimax lower bound for estimating entry $(i,j)$. 
        \begin{thm}
            \label{thm:minimax_max_path}
            For a deterministic sampling matrix $\Omega$ and a fixed pair of $(i,j)$, there exists an absolute constant $C>0$ such that 
            \begin{align*}
                \inf_{\hat M} \; \sup_{M^\ast} \; \mathbb{E} \left[ ( \hat M_{ij} - M_{ij}^\ast)^2 \right] \ge  \frac{C \sigma^2}{K}. 
            \end{align*}
        \end{thm}
        The dependence on $K$ (maximum number of disjoint paths) is the same in Theorem~\ref{thm:multi_path_est_upper_bound} and Theorem~\ref{thm:minimax_max_path}. When both $\sigma$ (noise level) and $L$ (max path length) are constants, the error upper bound matches the lower bound up to logarithmic factors. We prove Theorem~\ref{thm:minimax_max_path} by constructing hard instances using the minimum edge cut and relating it to the maximum number of paths $K$ via Menger's theorem.

        \begin{proof}[Proof of Theorem~\ref{thm:minimax_max_path}]
            In the following lemma, we construct two signal matrices such that the KL-divergence can be upper bounded by a multiple of $K$. 
            For a signal matrix $A^\ast$, let $f_{A^\ast}$ denote the data distribution: a product distribution of the densities of all observations, $f_{A^\ast}  = \otimes_{s,t} f_{A^\ast_{st}}$, where $f_{A^\ast_{st}}$ is a normal distribution $\normal(A^\ast_{st}, \sigma^2)$.  The proof of Lemma~\ref{lem:KL_divergence_max_path} is deferred toward the end.
            \begin{lem}
            \label{lem:KL_divergence_max_path}
                For some $0 < \varepsilon < 1$, there exist rank-$1$ signal matrices $A^\ast$ and $B^\ast$ such that $(A_{ij}^\ast - B_{ij}^\ast)^2 \le 4 \varepsilon^2$ and  $\KL(f_{A^\ast} \Vert f_{B^\ast}) \le \frac{2 K \varepsilon^2}{\sigma^2} $. 
            \end{lem}
            Applying Lemma~\ref{lem:KL_divergence_max_path}, we can derive the minimax lower bound by Le Cam's method.
        \end{proof}

        \begin{proof}[Proof of Lemma~\ref{lem:KL_divergence_max_path}]
            For a fixed $\Omega$, we construct signals $A^\ast$ and $B^\ast$ as rank-1 matrices:
            \begin{align*}
                A^\ast &= a b^\top \\
                B^\ast &= y z ^\top.
            \end{align*}
            For $a$ and $b$, we set them as $a = b = \varepsilon \cdot \boldsymbol{1}$ where  $\boldsymbol{1}$ represents an all-one vector. For $y$ and $z$, we define them according to a minimum edge cut for $u_i$ and $v_j$. 
            Suppose this cut partitions the nodes into two sets $\mathcal{L}$ and $\mathcal{R}$. We assume $u_i \in \mathcal{L}$ and $v_j \in \mathcal{R}$. 
            For node $s \in \mathcal{L}$, let $y_s = \varepsilon$ if $s$ is a left node (or $z_s = \varepsilon$ otherwise.) For node $t \in \mathcal{R}$, let $y_t = -\varepsilon$ if $t$ is a left node (or $z_t = -\varepsilon$ otherwise.) 
            For entry $(s,t)$, it is easy to see that $A_{st}^\ast = B_{st}^\ast$ if and only if $s$ and $t$ belong to the same set. To compute the KL-divergence $\KL(f_{A^\ast} \Vert f_{B^\ast} ) = \frac{1}{2\sigma^2} \sum \Omega_{st}  (A_{st}^\ast - B_{st}^\ast)^2$, we deduce that 
            \begin{align*}
                \Omega_{st} (A_{st}^\ast - B_{st}^\ast)^2 = 
                \begin{cases}
                    4 \varepsilon  & \Omega_{st}=1 \text{ and edge $(s,t)$ is in the cut} \\
                    0 & \text{otherwise.}
                \end{cases}
            \end{align*}
            By Menger's theorem, the size of the minimum cut is equal to the maximum number of edge-disjoint paths connecting $u_i$ and $v_j$. 
            Therefore, the desired result is proved.
        \end{proof}

    \subsection{Examples}

    In this section we discuss the implications of our bound in specific settings, discussing when the the guarantees are minimax optimal or potentially loose. 

    \subsubsection{Extreme Sparsity}
        We present an example that both motivates the usage of paths and showcases the effectiveness of our estimator. 
        The sampling matrix $\Omega \in \{0,1\}^{n \times n}$ only has support on the first row, the first column, and along the diagonal, excluding the $(1,1)$ entry. Namely, we have
        \begin{equation}
            \label{eq:example_sparsity_rank1}
            \Omega = 
            \begin{pmatrix}
                 0 & 1 & 1 & \cdots & 1 \\
                 1 & 1 & 0 & \cdots & 0 \\
                 1 & 0 & 1 &  & 0 \\
                 \vdots &  \vdots& & \ddots & \vdots \\
                 1 & 0 & 0 & \cdots &  1 
            \end{pmatrix}.
        \end{equation}
        Suppose we want to estimate $M_{11}^\ast$. Estimation is not easy since $\Omega$ is quite sparse. 
        Furthermore, we do not even have any densely observed and sufficiently large submatrix (which corresponds to a biclique in the graph) to apply methods introduced in~\cite{agarwal2021causal, xi2023nonuniform}. Nevertheless, estimation is not impossible because in the corresponding bipartite graph, information about $(u_1, v_1)$ can be deduced by paths from $u_1$ to $v_1$ with the form $u_1 \to v_i \to u_i \to v_1$, using precisely the observations given. 
        From a pure linear algebraic perspective, estimating the $(1, 1)$ entry amounts to figuring out the product $a_1^\ast b_1^\ast$. Observations on the first row and column are noisy versions of $(a_1^\ast b_2^\ast, \dots, a_1^\ast b_n^\ast)$ and $(a_2^\ast b_1^\ast, \dots, a_n^\ast b_1^\ast)$, respectively. Hence, we can estimate $a_1^\ast b_1^\ast$ by getting rid of the factors $b_2^\ast, \dots, b_n^\ast$ and $a_2^\ast, \dots, a_n^\ast$ using the diagonal observations. In this example, using paths connecting $u_1$ and $v_1$ is a natural way to obtain accurate estimation. Applying Theorem~\ref{thm:multi_path_est_upper_bound} and~\ref{thm:minimax_max_path}, we obtain matching upper and lower bounds \cyedit{as the minimum cut is equal to the number of disjoint paths, given by $n-1$}. 
        
    \subsubsection{Dense Submatrix}

 \begin{figure}[t]
    \centering
\includegraphics[width=0.8\textwidth]{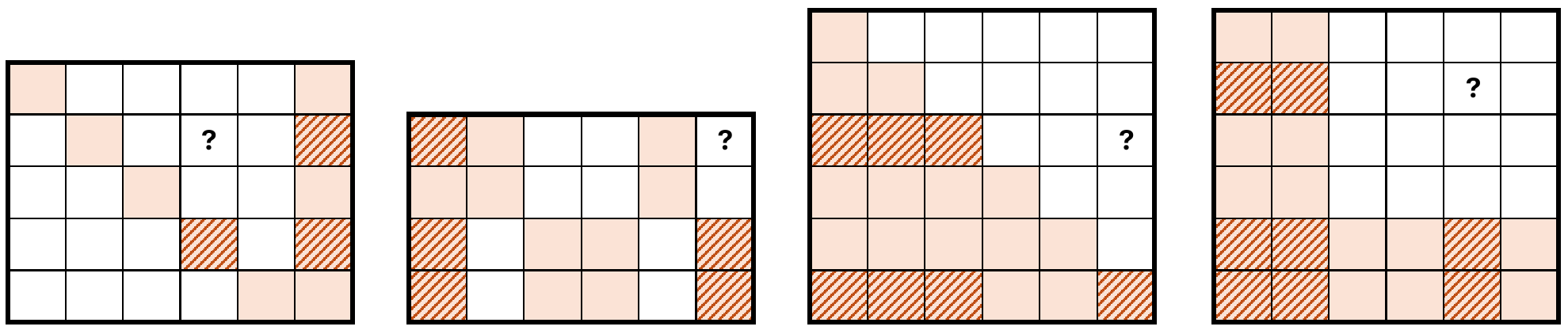}
    \caption{Examples of observation patterns that may naturally arise in recommendation systems panel data, and sequential decision making~\cite{agarwal2021causal}, where shaded regions indicate observed blocks of the matrix, and white regions indicate unobserved blocks. Suppose that we would like to estimate an entry in the block denoted with a ``{\bf ?}''. The striped regions show a dense observed submatrix containing the target entry, where all entries in the submatrix are observed except for the target entry itself.}
    \label{fig:dense_submatrix}
    \end{figure}
    
    The algorithm introduced in \cite{agarwal2021causal} relies on the presence of a dense submatrix containing the entry of interest. They argue that such structures naturally arises in applications such as recommendation systems panel data, and sequential decision making, where Figure \ref{fig:dense_submatrix} shows examples of such observation patterns. \cycomment{We could explain their motivation for these sampling patterns here if we wanted ... the examples correspond to recommendation systems, sequential decision making, staggered adoption, and panel data.} In all of these observation patterns, for any desired target entry $(i,j)$, there exists a submatrix defined by the row subset $\mathcal{I} \subseteq [n]$ and column subset $\mathcal{J} \subseteq [m]$ such that all the entries in the submatrix $\mathcal{I} \cup \{i\} \times \mathcal{J} \cup \{j\}$ are observed except for the target entry $(i,j)$ itself. As depicted in Figure~\ref{fig:bipartite_clique}, this dense submatrix in $\Omega$ results in a subgraph of $G(\Omega)$ where $i$ is fully connected to $\mathcal{J}$, $j$ is fully connected to $\mathcal{I}$, and there is a bipartite clique between $\mathcal{I}$ and $\mathcal{J}$. As a result, the maximum matching in the bipartite clique between $\mathcal{I}$ and $\mathcal{J}$ has size $K = \min(|\mathcal{I}|,|\mathcal{J}|)$, and thus we can use that matching to construct $K$ edge-disjoint paths of length 3 connecting $i$ and $j$. By applying Theorem~\ref{thm:multi_path_est_upper_bound}, this results in an high probability bound of 
    \[|\hat{M}_{ij} - M^*_{ij}| \lesssim \sigma^3 \|M^*\|_{\infty}^3 \sqrt{1/K},\]
    ignoring logarithmic factors. The minimum cut of the subgraph corresponding to the dense submatrix is also $K$, but this does not yet imply the bound is tight as the minimum cut could be larger on the entire graph. An upper bound on the minimum cut is the minimum between the degrees of $u_i$ and $v_j$ in the graph. In the majority of examples given in Figure \ref{fig:dense_submatrix}, the bound achieved by our algorithm is indeed minimax optimal.
    
    \begin{figure}[h]
    \centering
\includegraphics[width=0.45\textwidth]{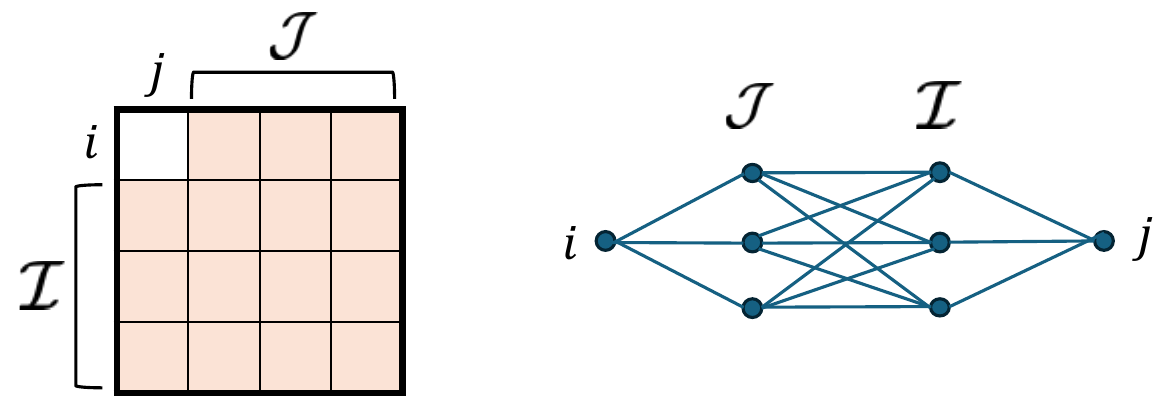}
    \caption{A dense submatrix $\mathcal{I} \cup \{i\} \times \mathcal{J} \cup \{j\}$ in $\Omega$ containing $i$ and $j$ can be used to construct $K$ disjoint paths of length 3 between $i$ and $j$ in the graph $\mathcal{G}(\Omega)$, where $K = \min(|\mathcal{I}|,|\mathcal{J}|)$.}
    \label{fig:bipartite_clique}
    \end{figure}

\subsubsection{Uniform Sampling}

Consider a uniform Bernoulli sampled observation pattern in which $(i,j) \in \Omega$ with probability $p$ independently for all $(i,j)$. For a fixed target entry $(i,j)$, the degree of $i$ and $j$ will concentrate around $pn$ for large $n$. Let $\mathcal{J}$ denote the neighbors of $i$ and let $\mathcal{I}$ denote the neighbors of $j$. Let $K$ denote the size of the maximum matching in the subgraph of $\mathcal{G}(\Omega)$ restricted to vertices $\mathcal{I}\cup\mathcal{J}$. Then the corresponding maximum matching can be used to construct $K$ disjoint paths of length 3, implying that our estimator achieves error bounded by 
\[|\hat{M}_{ij} - M^*_{ij}| \lesssim \sigma^3 \|M^*\|_{\infty}^3 \sqrt{1/K}.\]
A naive upper bound on the minimum cut is given by the degree of nodes $i$ and $j$, which is roughly $pn$. 
We can use the result that for $p = \omega(\log n/n)$, a random bipartite graph with vertex sets of size $n$ will have a perfect matching with probability tending to 1 wrt $n$, see Chapter 7 in \cite{Bollobas_2001}. Thus, our upper bound and lower bound will match if $p = \omega(\log pn/pn)$, i.e. $p = \omega(\log n/ \sqrt{n})$. While this allows for $p \to 0$, it is still fairly dense as the expected degree must grow as $pn = \omega(\sqrt{n} \log n)$. When the graph is sparser, then we would need a finer analysis that characterizes both the minimum cut of a random bipartite graph and the corresponding lengths of the shortest paths involved in the cut. As our upper bound degrades exponentially in the length of the path, we expect that our upper bound would be loose for sparse graphs where the path length may grow with $n$.

\section{Conclusion}

While entry specific matrix estimation is a daunting task in general, we show that for the special settings of additive matrices and rank-1 matrices we can obtain unbiased estimates for any entry $(i,j)$ as long as it is identifiable, i.e. the vertices $u_i$ and $v_j$ are connected in $\mathcal{G}(\Omega)$. We provide a network flow based perspective towards estimation under arbitrary observation patterns, where the connectivity properties of the graph governs the estimation accuracy through the effective resistances in the case of additive models and through the minimum cut in the case of rank-1 models. Our results provide fine grained bounds on how the estimation error of a specific entry in the matrix depends on an arbitrary unstructured observation pattern. While our results are tight for additive models, our upper and lower bounds in the rank-1 setting are not tight for observation patterns for which the connecting paths are growing in length with the size of the matrix, as is the case for sparse random graphs. Future directions include tightening the results for the rank-1 setting and extending our results to general low rank models.

\bibliographystyle{plain}
\bibliography{ref-merged}

\appendices

\section{Effective Resistance}
    \label{appen:effective_resistance}
    Consider a simple and connected graph $\cG = (\cV, \cE)$ with $|\cV|=n_v$ nodes and $|\cE|=n_e$ edges. We can construct an electrical network from $\cG$ by assigning unit resistance to each edge.
    The \emph{effective resistance}~\cite{klein1993er} is a quantity defined between two nodes $s$ and $t$, denoted as $R(s,t)$, that measures the potential difference between $s$ and $t$ when we apply one unit current between them. 
    It is a useful concept for gauging how connected $s$ and $t$ are. When we add (remove) paths between $s$ and $t$, the effective resistance $R(s,t)$ decreases (increases). In some sense, the effective resistance captures how similar or how close the nodes are, which is why it is also referred to as the \emph{resistance distance}. 
    It turns out that the effective resistance is closely connected to a bunch of concepts: the graph Laplacian, spanning trees, and random walks. We refer interested readers to a survey~\cite{vos2016methods}. To give an example of such connections, we can consider a simple random walk on $\cG$ with equal transition probability.\footnote{If we are at node $u$, we transition to a neighboring node with probability $1 / \deg(u)$, where $\deg(u)$ is the degree of $u$.} In this scenario, the effective resistance $R(s,t)$ is proportional to the commute time between $s$ and $t$, namely the expected time for a random walk starting in $s$, to reach $t$ and return to $s$. In fact, the commute time between $s$ and $t$ is $2 n_e R(s,t)$, according to~\cite{chandra1989electrical}. 
    This interpretation again emphasizes that the effective resistance measures how connected the nodes are: when there are more paths between $s$ and $t$, the smaller the effective resistance and it takes less time to travel between $s$ and $t$ for a random walk. 

    We can compute the effective resistance $R(s,t)$ using the Laplacian matrix $L \in \R^{n_v \times n_v}$, defined as
    \begin{align*}
        L_{uv} = 
        \begin{cases}
            \deg(u) & u=v \\
            -1 & u \neq v \text{ and $u$ is adjacent to $v$} \\
            0 & \text{otherwise.}
        \end{cases}
    \end{align*}
    Equivalently, we can express $L$ as $L = D - A$, where $D$ is the degree matrix and $A$ is the adjacency matrix.
    The Laplacian is symmetric and not invertible. By definition, the sum of each row in $L$ is zero, namely $L \one = \zero$. Hence, $\one$ is an eigenvector of $L$ corresponding to eigenvalue zero. Since we assume $\cG$ is connected, we can show that all other eigenvalues of $L$ are non-zero. 
    Let the Moore--Penrose pseudo-inverse of $L$ be denoted as $L^+.$
    The pseudo-inverse can be computed by leveraging singular value decomposition (SVD). Suppose the SVD of $L$ has the form $L = U \Sigma V^\top$. The pseudo-inverse is $L^+ = V \Sigma^+ U^\top$, where $\Sigma^+$, the pseudo-inverse of $\Sigma$ can be obtained by substituting the non-zero diagonal values with their reciprocals. 
    The Laplacian is related to the oriented incidence matrix (a.k.a.~signed edge-node adjacency matrix) $B \in \R^{n_e \times n_v}$, defined as
    \begin{align*}
        B_{(u,v), w} = 
        \begin{cases}
            1 & w=u \\
            -1 & w=v \\
            0 & \text{otherwise.}
        \end{cases}
    \end{align*}
    By definition, we have $L = B^\top B$.  
    With unit resistance $r_{e}=1$ for all $e\in\cE$, the voltage vector $\vv \in \R^{n_v}$ and the current $\ii \in \R^{n_e}$ satisfy $i_{(x,y)} = \vv_x - \vv_y$, for $(x,y) \in \cE$ 
    by Ohm's Law. 
    Using the incidence matrix $B$, we have $\ii = B \vv$. 
    By flow conservation, we have $\vv = L^{+} (e_{s} - e_{t})$, where $e_s$ and $e_{t}$ denote the standard basis vectors. 
    Therefore, the current and the effective resistance satisfy
    \begin{align}
        \ii &= BL^+ (e_s - e_t) \label{eq:electrical_flow_defn} \\
        R(s,t) &= (e_s - e_t)^\top L^+ (e_s - e_t). \label{eq:effective_resistance_defn}
    \end{align}
    
    The \emph{energy} of a flow $f \in \R^{n_e}$ is defined as $\sum_{(u,v) \in \cE} f^2_{(u,v)}.$ Thomson's principle shows that among all unit flows $f$ from $s$ to $t$, the current is the unique one that minimizes the energy of the flow. 
    \begin{thm}[Thomson's Principle]
    \label{thm:thomson}
          For an arbitrary unit flow $f$, we have $\sum_{(u,v) \in \cE} f^2_{u,v} \ge R(s,t)$, with equality if and only if $f = \ii$.
    \end{thm}
    
\section{Linear Regression: Dual Variables \& Minimax Optimality}
\label{appen:linear_regression}
    We have mentioned that the technique used to prove Theorem~\ref{thm:minimax_effective_resistance} can be further extended to a general problem of linear regression. 
    Specifically, we will show that when proving minimax optimality, the hard instances can be constructed using the dual variables. 
    In fact, the voltage vector is the dual variable of the energy-minimizing (variance-minimizing) program. We first show this and then discuss linear regression.
    Consider the following program
        \begin{equation}
        \label{eq:flow_min_var_optimization_copy}
            \begin{aligned}
                \min_{f} \; & \frac12 \sum_{k, \ell } f_{k, n+\ell}^2 \\
            \text{s.t. } & \sum_{\ell} f_{i, n+\ell} = 1 \\
            & -\sum_{k} f_{k, n+j} = -1 \\
                        & \sum_{\ell} f_{i', n+\ell} = 0, \quad \forall i' \neq i \\
                        & - \sum_{k} f_{k,n+j'} = 0, \quad \forall j' 
                        \neq j.
            \end{aligned}
        \end{equation}
        The equality constraints in~\eqref{eq:flow_min_var_optimization_copy} can be written as
        \begin{align*}
            B^\top f = e_{i} - e_{n+j},
        \end{align*}
        where $B \in \R^{n_e \times n_v}$ is the incidence matrix of the graph. 
        Consider the primal problem
        \begin{align*}
            \min_{f} \; & \frac12 \sum_{k, \ell } f_{k, n+\ell}^2 \\
            \text{s.t. } & -B^\top f = b = e_{n+j} - e_i.
        \end{align*}
        Define the dual variables as $\lambda \in \R^{n_v}$. The Lagrangian is 
        \begin{align*}
            L(f, \lambda ) = \frac12 f^\top f + \lambda^\top (-B^\top f - b).
        \end{align*}
        The dual function is
        \begin{align*}
            g(\lambda) &= \inf_{f} L(f,\lambda) \\
                        &= L(B\lambda, \lambda) \\
                        &= -\frac12 \lambda^\top B^\top B \lambda - \lambda^\top b.
        \end{align*}
        The dual problem is
        \begin{align*}
            \max_{\lambda} \; -\frac12 \lambda^\top B^\top B \lambda - \lambda^\top b.
        \end{align*}
        One solution is $\lambda^\ast = - (B^\top B)^{+} b = L^+(e_i - e_{n+j})$, which is the voltage vector. We actually have strong duality here since the dual maximum is equal to $ \frac12 b^\top L^+ b$ (one-half of the effective resistance), which is equal to the primal minimum.

        Now consider the problem of linear regression. The unknown parameter is $\beta \in \R^p$. The observations are given by $y_i = x_i^\top \beta + \varepsilon_i$, for all $i\in[n]$. The covariates $x_i \in \R^p, i\in[n]$ are known and the noise terms $\varepsilon_i$ are i.i.d.~$\normal(0,\sigma^2)$. Assume $X \in \R^{n \times p}$ is full rank.
        Suppose we want to estimate $\tau = w^\top \beta$, a linear function of $\beta$. 
        Define a linear combination of the observations as $\hat \tau = f^\top y$, where $f\in\R^n$ is a weight vector to be determined. If $f$ satisfies $X^\top f = w $, then $\hat \tau$ is an unbiased estimator of $\tau$. In addition, if we want to minimize the variance of $\hat \tau$, we obtain the following optimization problem
        \begin{align*}
            \min_f \; & \frac12 f^\top f \\
            \text{s.t. } & X^\top f = w .
        \end{align*}
        The optimal solution is $f^\ast = X (X^ \top X)^{-1} w$ and the corresponding $\hat \tau$ is $w^\top (X^\top X)^{-1} X^\top y$, which aligns with the least squares estimator. 
        Define the dual variables as $\lambda \in \R^p$. The Lagrangian is
        \begin{align*}
            L(f,\lambda) = \frac12 f^\top f + \lambda^\top (w - X^\top f ).
        \end{align*}
        The dual function is
        \begin{align*}
            g(\lambda) &= \inf_f L(f, \lambda) \\
            &= L(X \lambda, \lambda) \\
            &= - \frac12 \lambda^\top X^\top X\lambda + \lambda^\top w.
        \end{align*}
        The dual problem is
        \begin{align*}
            \max_\lambda \; -\frac12 \lambda^\top X^\top X\lambda + \lambda^\top w.
        \end{align*}
        The above problem has a unique solution $\lambda^\ast = (X^\top X)^{-1} w$ with $g(\lambda^\ast) = \frac12 w^\top (X^\top X)^{-1} w$.
    
        Next, we show the least squares estimator is minimax optimal.
    
        \begin{lem}
        \label{lem:KL_divergence_dual_LR}
            For some $0 < \epsilon < 1$, there exist $\beta^0$ and $\beta^1$ such that $\vert w^\top ( \beta^0 - \beta^1 ) \vert^2  = 4 \epsilon^2 g^2(\lambda^\ast)$and  $\KL(Y^0 \Vert Y^1)  = \frac{\epsilon^2 g(\lambda^\ast)}{\sigma^2} $, where $Y^0 = X \beta^0 + \varepsilon$ and $Y^1 = X \beta^1 + \varepsilon$. 
        \end{lem}
        Applying Lemma~\ref{lem:KL_divergence_dual_LR} and Le Cam's method, we have
        \begin{align*}
            \inf_{\hat \beta} \sup_{\beta} \expt{ \vert w^\top (\hat \beta - \beta )\vert ^2} \ge \frac{\epsilon^2 g^2 (\lambda^\ast)}{2} \left(1 - \sqrt{\frac{\epsilon^2 g(\lambda^\ast) }{2\sigma^2}} \right).
        \end{align*}
        Setting $\epsilon = \frac{2 \sqrt 2 \sigma}{3\sqrt{g(\lambda^\ast)}}$ yields
        \begin{align*}
            \inf_{\hat \beta} \sup_{\beta} \expt{ \vert w^\top (\hat \beta - \beta )\vert ^2} \ge \frac{4 \sigma^2 g(\lambda^\ast)}{27}.
        \end{align*}
    
        \begin{proof}[Proof of Lemma~\ref{lem:KL_divergence_dual_LR}]
            Let $\beta^0 = \zero$ and $\beta^1 = \epsilon \lambda^\ast$. Since $Y^0 \sim \normal(X \beta^0, \sigma^2 I)$ and $Y^1 \sim \normal(X \beta^1, \sigma^2 I)$, we have
            \begin{align*}
                \KL(Y^0 \Vert Y^1) &= \frac1{2\sigma^2} \norm{X (\beta^0 - \beta^1) }_2^2 \\
                &= \frac{\epsilon^2}{2\sigma^2} \lambda^{\ast \top} X^\top X \lambda^\ast \\
                &= \frac{\epsilon^2}{2\sigma^2} w^\top (X^\top X)^{-1} w \\
                &= \frac{\epsilon^2}{\sigma^2} g(\lambda^\ast).
            \end{align*}
        \end{proof}
    

\section{Proofs}
\label{sec:appendix_proofs}

\subsection{Proof of Lemma~\ref{lem:closed-form_lse}}
\label{appen:proof_closed-form_lse}
In the graph $\cG(\Omega)$, define $N_i = \{j \in [m]: \Omega_{ij}=1\}$ and as the neighbors of $u_i$; similarly, define $N^j = \{ i \in [n] : \Omega_{ij}=1 \}$ as the neighbors of $v_j$. 
The first derivatives of $f$ are
\begin{align*}
\frac{\partial }{\partial a_i} f(\aa,\bb) &= \sum_{j \in N_i} (a_i + b_j -M_{ij}), \quad \forall i \\
\frac{\partial }{\partial b_j} f(\aa,\bb) &= \sum_{i \in N^j} (a_i + b_j -M_{ij}), \quad \forall j.
\end{align*}
Setting the derivatives to zero, we obtain
\begin{align*}
\sum_{j \in N_i} (a_i + b_j -M_{ij}) & = 0 , \quad \forall i \\
\sum_{i \in N^j} (a_i + b_j -M_{ij}) & = 0 , \quad \forall j. \\
\end{align*}
Hence, the solution to~\eqref{eq:lse_factor_defn} satisfies
\begin{align*}
W(\Omega)
\begin{pmatrix}
    \hat \aa \\ \hat  \bb
\end{pmatrix}
 = \begin{pmatrix}
    \diag(\Omega M^\top) \\
    \diag(\Omega^\top M)
\end{pmatrix},
\end{align*}
where $W \in \R^{n_v \times n_v}$ is defined as
\begin{align*}
W = 
\begin{pmatrix}
    D_U & \Omega \\
    \Omega^\top & D_V
\end{pmatrix},
\end{align*}
where $D_U$ and $D_V$ denote diagonal matrices with diagonals being the degrees of the left nodes and the right ones, respectively. 
The minimum norm solution is therefore 
$W^+(\Omega) \begin{pmatrix}
    \diag(\Omega M^\top) \\
    \diag(\Omega^\top M)
\end{pmatrix}$.

The last step is to show $W^+(\Omega) = \begin{pmatrix}
    \Gamma_{11} & -\Gamma_{12} \\
    -\Gamma_{21} & \Gamma_{22}
\end{pmatrix}$.
Using the formula from p.101 of~\cite{gentle2007matrix}, we have
\begin{align*}
L^{+}(\Omega) & = 
    \begin{pmatrix}
        D_U^{+} + D_U^{+} \Omega \left( D_V - \Omega^\top D_U^{+} \Omega  \right)^{+} \Omega^\top D_U^{+} & D_U^{+} \Omega \left( D_V - \Omega^\top D_U^{+} \Omega  \right)^{+}  \\
        \left( D_V - \Omega^\top D_U^{+} \Omega  \right)^{+} \Omega^\top D_U^{+} & \left( D_V - \Omega^\top D_U^{+} \Omega  \right)^{+} 
    \end{pmatrix} \\
W^+(\Omega) &= \begin{pmatrix}
        D_U^{+} + D_U^{+} \Omega \left( D_V - \Omega^\top D_U^{+} \Omega  \right)^{+} \Omega^\top D_U^{+} & -D_U^{+} \Omega \left( D_V - \Omega^\top D_U^{+} \Omega  \right)^{+}  \\
        -\left( D_V - \Omega^\top D_U^{+} \Omega  \right)^{+} \Omega^\top D_U^{+} & \left( D_V - \Omega^\top D_U^{+} \Omega  \right)^{+} 
    \end{pmatrix}.
\end{align*}
Hence, we obtain the desired result.

\subsection{Proof of Theorem~\ref{thm:equivalence_efe_lse}}
\label{appen:proof_equivalence_efe_lse}
Fix $(i,j) \in [n] \times [m]$. Applying Lemma~\ref{lem:closed-form_lse}, we deduce that the least squares estimator satisfies
\begin{align*}
\hat{M}^{\lse}_{ij} &= \hat\aa_i + \hat\bb_j \\
&= e_i^\top  \Gamma_{11}  \diag(\Omega M^\top) - e_i^\top \Gamma_{12}  \diag(\Omega^\top M)  - e_j^\top \Gamma_{21}  \diag(\Omega M^\top) + e_j^\top \Gamma_{22}  \diag(\Omega^\top M)\\
&= ( \Gamma_{11} e_i  -  \Gamma_{12} e_j )^\top \diag(\Omega M^\top)  - ( \Gamma_{21} e_i - \Gamma_{22} e_j )^\top \diag(\Omega^\top M) \\
&= \begin{pmatrix}
    \diag(\Omega M^\top)^\top & \diag(\Omega^\top M)^\top
\end{pmatrix}
\begin{pmatrix}
    \Gamma_{11} e_i - \Gamma_{12}e_j \\
    \Gamma_{21} e_i - \Gamma_{22} e_j
\end{pmatrix} \\
&= \begin{pmatrix}
    \diag(\Omega M^\top)^\top & \diag(\Omega^\top M)^\top
\end{pmatrix}
\begin{pmatrix}
    \Gamma_{11} & \Gamma_{12} \\
    \Gamma_{21} & \Gamma_{22}
\end{pmatrix}
(e_i - e_{n+j}) \\
&= \begin{pmatrix}
    \diag(\Omega M^\top)^\top & \diag(\Omega^\top M)^\top
\end{pmatrix} L^+(\Omega) (e_i - e_{n+j}).
\end{align*}
On the other hand, by~\eqref{eq:closed_form_flow_est}, the optimal flow-based estimator has the expression
\begin{align*}
\hat{M}_{ij}^{\efe} = (B^\top(\Omega) \vec_{\Omega} (M) )^\top L^+(\Omega) (e_i - e_{n+j}).
\end{align*}
By definition, we get $B^\top(\Omega) \vec_{\Omega} (M) = \begin{pmatrix}
    \diag(\Omega M^\top) \\
    \diag(\Omega^\top M)
\end{pmatrix}$, which completes the proof.

\subsection{Proof of Lemma~\ref{lem:concentration_num_denom}}
\begin{proof}[Proof of Lemma~\ref{lem:concentration_num_denom}]
We will only show the proof for the second inequality since the proof for the first one is almost identical. 
Since the paths are disjoint, it is easy to see that the random variables $h_1(Y), \dots, h_K(Y)$ are independent. 
For $\delta \in (0,1)$, define truncated random variables $Z_{ij} = Y_{ij} \indic_{\{|Y
_{ij} - M^\ast_{ij}| \le \sigma \sqrt{2 \log(nm / \delta)} \}}$ for all $(i,j) \in [n] \times [m]$. 
By definition, random variables $h_1(Z), \dots, h_K(Z)$ are also independent. 
We can show that $Z_{ij}=Y_{ij}, \forall (i,j)$ with high probability. 
The union bound and the standard Gaussian tail bound yields 
\begin{align*}
\prob{ \cup_{ (i,j) \in [n] \times [m] } \left\{ |Y_{ij} - M_{ij}^\ast | > \sigma \sqrt{2 \log(nm/\delta)} \right\} } & \le  \sum_{(i,j)} \prob{ |Y_{ij} - M_{ij}^\ast| / \sigma >  \sqrt{2 \log (nm/\delta) } } \\
&\le 2nme^{- \log(nm/\delta)} \\
&= 2\delta.
\end{align*}
By triangle inequality, we derive that 
\begin{align*}
& \card{\frac1K \sum_{k \in [K]} h_k(Y) - \frac1K \sum_{k \in [K]} 
\expect{h_k(Y)}} \\ 
\le  & \card{ \frac1K \sum_{k \in [K]} h_k(Y) - \frac1K \sum_{k \in [K]} h_k(Z)} 
 + \card{\frac1K \sum_{k \in [K]} \expect{h_k(Y)}  - \frac1K \sum_{k \in [K]} \expect{h_k(Z)}  }  \\
 & + \card{\frac1K \sum_{k \in [K]} h_k(Z) - \frac1K \sum_{k \in [K]} 
\expect{h_k(Z)}}.
\end{align*}
We have already shown that $\card{ \frac1K \sum_{k \in [K]} h_k(Y) - \frac1K \sum_{k \in [K]} h_k(Z)} = 0$ with probability at least $1-2\delta$.
Since $Z_{ij}$'s are bounded,  the random variables $h_k(Z)$ are also bounded: $|h_k(Z)| \le  ( \sigma \sqrt{2 \log (nm/\delta)} )^{L}$. An application of Hoeffding's inequality yields:
\begin{align*}
 \card{\frac1K \sum_{k \in [K]} h_k(Z) - \frac1K \sum_{k \in [K]}  \expect{h_k(Z)} } \le C \sigma^L \sqrt{\frac{ 2^L \log^{L+1}{(nm/\delta) }}{K}},
\end{align*}
with probability at least $1 - \delta$.
For each $k$, the difference in expectation is
\begin{align*}
&\expect{h_k(Y) - h_k(Z)} \\
=  &\frac{2 \sigma^{\ell_k-1}}{\sqrt{2\pi}} \Pi_{s=3,5,\dots, \ell_k}  \int_{\sqrt{ 2 \log(nm/\delta) }}^{\infty} (t + M_{x_s^k x_{s-1}^{k}}^\ast)^2 e^{-t^2/2} \ddup t \\
= &\frac{2\sigma^{\ell_k-1}}{\sqrt{2\pi}} \Pi_{k=3,5,\dots, \ell_k}  \left[ 
      \int_{\sqrt{2\log(nm/\delta)}}^{\infty} t^2 e^{-t^2/2} \ddup t  +  2 M_{x_s^k x_s^{k+1}}^\ast \int_{\sqrt{2\log(nm/\delta)}}^{\infty} t e^{-t^2/2} \ddup t  + (M_{x_s^k x_s^{k+1}}^\ast )^2 \int_{\sqrt{2\log(nm/\delta)}}^{\infty} e^{-t^2/2} \ddup t  \ \right] \\
\end{align*}
Note that $\int_{\sqrt{2\log(nm/\delta)}}^{\infty} e^{-t^2/2} \ddup t \le \frac{\delta}{nm}$, by Gaussian tail bound. 
We also have $\int_{\sqrt{2\log(nm/\delta)}}^{\infty} t e^{-t^2/2} \ddup t = - e^{-t^2/2} \bigg\vert_{\sqrt{2\log(nm/\delta)}}^{\infty} = \frac{\delta}{nm}.$
In addition, by integration by parts, we have
\begin{align*}
\int_{\sqrt{2\log(nm/\delta)}}^{\infty} t^2 e^{-t^2/2} \ddup t &= -\int_{\sqrt{2\log(nm/\delta)}}^{\infty} t \ddup(e^{-t^2/2}) \\
&= \frac{\delta \sqrt{2\log(nm/\delta)}}{nm} + \int_{\sqrt{2\log(nm/\delta)}}^{\infty} e^{-t^2/2} \ddup t \\
&\le \frac{3 \delta \sqrt{\log(nm/\delta)}}{nm}.
\end{align*}
Let $\inftynorm{M^\ast} = \max_{ij} |M_{ij}^\ast|$. 
We derive that
\begin{align*}
\card{\frac1K \sum_{k \in [K]} \expect{g_k(Y)}  - \frac1K \sum_{k \in [K]} \expect{g_k(Z)} } & \le \frac{C \sigma^L (1 + \inftynorm{M^\ast}^L) \sqrt{\log(nm/\delta)} }{nm}.
\end{align*}
Combining pieces, we obtain that with probability at least $1 - 3\delta$, the following inequality holds
\begin{align*}
\card{\frac1K \sum_{k \in [K]} g_k(Y) - \frac1K \sum_{k \in [K]} 
\expect{g_k(Y)}} \le C \sigma^L (1 + \inftynorm{M^\ast}^L)  \sqrt{\frac{ 2^L \log^{L+1}{(nm/\delta) }}{K}},
\end{align*}
where we use the fact that $K \le nm$.
\end{proof}

\section{Algorithm for Finding Disjoint Paths}
\label{appen:path_finder}
The following algorithm is from~\cite{kleinberg2006algorithm}. Let $\cE$ denote the edge set of graph $\cG$. In line 3, we use $c \equiv 1$ to denote a unit capacity on every edge of $\cG'$. Lines 4--9 eliminate the possibility that we have flow on both $e=(u,v)$ and $e'=(v,u)$. 

    \begin{algorithm}
         \SetAlgoLined
         \SetKwFunction{algo}{PathFinder}{}
         \SetKwProg{myalg}{Function}{}{}
         \myalg{\algo{$\cG, i, j$}}{
            Construct a directed graph $\cG'$ by adding two edges $(u,v)$ and $(v,u)$ if $u$ and $v$ are adjacent in $\cG$. 
            Delete the edges into $u_i$ and out of $v_j$ from $\cG'$. \\
            Obtain maximum flow of value $K$: $f \gets \texttt{Ford-Fulkerson}(\cG', s=u_i, t=v_j, c\equiv 1)$.\\
            \For{$e =(u,v) \in \cE$}{
                Let $e' \gets (v,u)$. \\
                \If{$f(e) \neq 0, f(e') \neq 0$}{
                    Let $\delta \gets \min\{f(e), f(e')\}$.\\
                    Update $f(e) \gets f(e) - \delta, f(e') \gets f(e') - \delta$.\\
                }
            }
            \For{$k \in [K]$}{
                Start walking from $u_i$ by following $f$.\\
                \uIf{reach $v_j$}{
                    Save the path as $\path(\xx^k, \ell_k)$.\\
                    Decrease the flow $f$ on the edges along the path to $0$.
                }
                \uElseIf{encounter a cycle}{
                    Decrease the flow $f$ on the edges along the cycle to $0$.\\
                    Start the walk again.
                }
            }
            \KwRet{$\{\path(\xx^k, \ell_k)\}_{k\in[K]}$}
         }
         \caption{Find disjoint paths in a bipartite graph. \label{algo:path_finder}}
    \end{algorithm}

\end{document}